\documentclass{article}

    \usepackage[preprint, nonatbib]{neurips_2020}

\usepackage[utf8]{inputenc} 
\usepackage[T1]{fontenc}    
\usepackage{hyperref}       
\usepackage{url}            
\usepackage{booktabs}       
\usepackage{amsfonts}       
\usepackage{nicefrac}       
\usepackage{microtype} 
\usepackage{footmisc}
\usepackage{graphicx}
\usepackage{subfigure}
\usepackage{booktabs}
\usepackage{multirow}
\usepackage{algorithm}
\usepackage{algorithmic}
\usepackage{amsmath}

\graphicspath{{figures/}}
\input{Definitions}
\title{Towards Assessment of Randomized Smoothing Mechanisms for Certifying Adversarial Robustness}

\author{
Tianhang Zheng,\textsuperscript{\rm 1}\thanks{The first two authors contributed equally to this work.}~~Di Wang,\textsuperscript{\rm 2,3}
Baochun Li,\textsuperscript{\rm 1} 
Jinhui Xu\textsuperscript{\rm 2}\\
\textsuperscript{\rm 1} University of Toronto \\
\textsuperscript{\rm 2}
State University of New York at Buffalo \\
\textsuperscript{\rm 3} King Abdullah University of Science and Technology
}

\begin{document}

\maketitle

\begin{abstract}
As a certified defensive technique, randomized smoothing has received considerable attention due to its scalability to large datasets and neural networks. However, several important questions remain unanswered, such as (i) whether the Gaussian mechanism is an appropriate option for certifying $\ell_2$-norm robustness, and (ii) whether there is an appropriate randomized (smoothing) mechanism to certify $\ell_\infty$-norm robustness. To shed light on these questions, we argue that the main difficulty is how to assess the appropriateness of each randomized mechanism. In this paper, we propose a generic framework that connects the existing frameworks in \cite{lecuyer2018certified, li2019certified}, to assess randomized mechanisms. Under our framework, for a randomized mechanism that can certify a certain extent of robustness, we define the magnitude of its required additive noise as the metric for assessing its appropriateness. We also prove lower bounds on this metric for the $\ell_2$-norm and $\ell_\infty$-norm cases as the criteria for assessment. Based on our framework, we assess the Gaussian and Exponential mechanisms by comparing the magnitude of additive noise required by these mechanisms and the lower bounds (criteria). We first conclude that the Gaussian mechanism is indeed an appropriate option to certify $\ell_2$-norm robustness. Surprisingly, we show that the Gaussian mechanism is also an appropriate option for certifying $\ell_\infty$-norm robustness, instead of the Exponential mechanism. Finally, we generalize our framework to $\ell_p$-norm for any $p\geq2$. Our theoretical findings are verified by evaluations on CIFAR10 and ImageNet. 
\end{abstract}
\vspace{-0.25in}
\section{Introduction}
\vspace{-0.1in}
The past decade has witnessed tremendous success of deep learning 
in handling various learning tasks like image classification \cite{krizhevsky2012imagenet}, natural language processing \cite{cho2014learning}, and game 
playing \cite{silver2016mastering}. 
Nevertheless, a major unresolved issue of deep learning is its vulnerability to adversarial samples, which are almost indistinguishable from natural samples to humans but can mislead deep neural networks (DNNs) to make wrong predictions with high confidence \cite{szegedy2013intriguing,goodfellow2014explaining}. This phenomenon, referred to as adversarial attack, is considered to be one of the most significant threats to the deployment of many deep learning systems.
Thus, a substantial amount of effort has been devoted to developing defensive techniques against it. However, a majority of existing defenses are of heuristic nature ({\em i.e.,} without any theoretical guarantees), implying that they may be ineffective against stronger attacks. 
Recent work \cite{he2017adversarial, athalye2018obfuscated, uesato2018adversarial} has confirmed this concern by showing that most of those heuristic defenses actually fail to defend against strong adaptive attacks. This forces us to shift our attention to certifiable defenses as they can classify all the samples in a predefined neighborhood of the natural samples with a theoretically-guaranteed error bound. Among all the existing certifiable defensive techniques, randomized smoothing is becoming increasingly popular due to its scalability to large datasets and arbitrary networks. 
\cite{lecuyer2018certified} first relates adversarial robustness to differential privacy, and proves that adding noise is a certifiable defense against adversarial examples. \cite{li2019certified} connects adversarial robustness with the concept of R\'{e}nyi divergence, and improves the estimate on the lower bounds of the robust radii. Recently, \cite{cohen2019certified} successfully certifies $49\%$ accuracy on the original ImageNet dataset under adversarial perturbations with $\ell_2$ norm less than $0.5$. 

Despite these successes, there are still several unanswered questions regarding randomized (smoothing) mechanisms\footnote{In this paper, ``randomized mechanism'' is an abbreviation for ``randomized smoothing mechanism''.}. One such question is, why should we use the Gaussian mechanism for randomized smoothing to certify $\ell_2$-norm robustness, or is there any more appropriate mechanism than the Gaussian mechanism? Another important question is regarding the ability of this method to certify $\ell_\infty$-norm robustness. If randomized smoothing can certify $\ell_\infty$-norm robustness, what mechanism is an appropriate choice? All these questions motivate us to develop a framework to assess the appropriateness of a randomized mechanism for certifying $\ell_p$-norm robustness. 

In this paper, we take a promising step towards answering the above questions by proposing a generic and self-contained framework, {\em which applies to different norms and connects the existing frameworks in \cite{lecuyer2018certified, li2019certified}}, for assessing randomized mechanisms. Our framework employs the Maximal Relative R\'{e}nyi (MR) divergence as the probability distance measurement, and thus, the definition of robustness under this measurement is referred to as $D_{MR}$ robustness.
Under our framework, we define the magnitude ({\em i.e.,} the expected $\ell_\infty$-norm) of the noise required by a mechanism to certify a certain extent of robustness as the metric for assessing the appropriateness of this mechanism. To be specific, a more ``appropriate'' randomized mechanism under this definition refers to a mechanism that can certify a certain extent of robustness with ``less'' additive noise. Given this definition, it is natural to define the assessment criteria as the lower bounds on the magnitude of the noise required to certify $\ell_p$-norm ($D_{MR}$) robustness, in that we can judge whether a mechanism is an appropriate option based on the gap between the magnitude of noise needed by the mechanism and the lower bounds. 

Inspired by the theory regarding the lower bounds on the sample complexity to estimate one-way marginals in differential privacy, we prove lower bounds on the magnitude of additive noise required by any randomized smoothing mechanism to certify $\ell_2$-norm and $\ell_\infty$-norm $D_{MR}$ robustness. We demonstrate the gap between the magnitude of Gaussian noise required by the Gaussian mechanism and the lower bounds is only $O(\sqrt{\log d})$ for both $\ell_2$-norm and $\ell_\infty$-norm, where $d$ is the dimensionality of the data. Note that this gap is small for datasets like CIFAR-10 and ImageNet, which indicates that the Gaussian mechanism is an appropriate option for certifying $\ell_2$-norm or $\ell_\infty$-norm robustness. Moreover, we also show that the Exponential mechanism is not an appropriate option for certifying $\ell_\infty$-norm robustness since the gap scales in $O(\sqrt{d})$.
To summarize, our contribution is four-fold:
\vspace{-0.06in}
\begin{itemize}
    \setlength\itemsep{0.2em}
    \item We propose a generic and self-contained framework for assessing randomized mechanisms, which applies to different norms and connects the existing frameworks such as \cite{lecuyer2018certified} and \cite{li2019certified}.
    \item We define a metric for assessing randomized mechanisms, {\em i.e.,}
    the magnitude of the additive noise to certify adversarial robustness, and we derive the lower bounds on the magnitude of the additive noise to certify $\ell_2$-norm and $\ell_\infty$-norm robustness as the criteria for the assessment. Also, we extend this assessment framework to $\ell_p$-norm for any $p\geq2$ in Appendix, which indicates the curse of dimensionality on randomized smoothing for certifying $\ell_p$-norm ($p>2$) robustness.
    \item We assess the Gaussian mechanism and the Exponential mechanism based on the metric and the lower bounds ({\em i.e.,} the criteria). 
  We conclude that the Gaussian mechanism is an appropriate option to certify $\ell_2$-norm and $\ell_\infty$-norm robustness, and the Exponential mechanism is not an appropriate option for certifying $\ell_\infty$-norm robustness here.
  \item We conduct extensive evaluations to verify our theoretical findings on our framework (with the same certified robust radii as in \cite{li2019certified}) and Cohen et al.'s framework \cite{cohen2019certified}.
\end{itemize}
\vspace{-0.06in}
{\em Due to the space limit, all the omitted proofs and some experimental results are included in Appendix (in the supplementary material).}
\vspace{-0.1in}
\section{Related Work}\label{sec:related}
\vspace{-0.1in}
To our knowledge, there are mainly three approaches to certify adversarial robustness standing out. 
The first approach formulates the task of adversarial verification as an optimization problem and solves it by tools like convex relaxations and duality \cite{dvijotham2018dual, raghunathan2018certified, wong2018provable}. Given a convex set (usually an $\ell_p$ ball) as input, the second approach maintains an outer approximation of all the possible outputs at each layer by various techniques, such as interval bound propagation, hybrid zonotope, abstract interpretations, and etc. \cite{mirman2018differentiable, wang2018efficient, gowal2018effectiveness, zhang2019towards, balunovic2020adversarial}.
The third approach uses randomized smoothing to certify robustness, which is {\em the main focus} of this paper.
Randomized smoothing for certifying robustness becomes increasingly popular due to its strong scalability to large datasets and arbitrary networks \cite{lecuyer2018certified, li2018second, cohen2019certified, dvijotham2018dual, salman2019provably}. 
For this approach, 
\cite{lecuyer2018certified} first proves that randomized smoothing can certify the $\ell_2$ and $\ell_1$-norm robustness using the differential privacy theory. \cite{li2018second} derives a tighter lower bound on the $\ell_2$-norm robust radius based on a lemma on R\'{e}nyi divergence. \cite{cohen2019certified} further obtains a tight guarantee on the $\ell_2$-norm robustness using the Neyman-Pearson lemma. 
\cite{dvijotham2020framework} proposes a new framework based on f-divergence that applies to different measures. \cite{salman2019provably} combines \cite{cohen2019certified} with adversarial training, and \cite{jia2019certified} extends the method in \cite{cohen2019certified} to the top-k classification setting.
We note that, compared with  \cite{cohen2019certified}, the frameworks proposed in \cite{lecuyer2018certified, li2019certified} are more general.  
In the following, we briefly review the basic definitions and theorems in the frameworks of \cite{lecuyer2018certified, li2019certified}, which helps us demonstrate the inherent connections between our framework and those two frameworks \cite{lecuyer2018certified, li2019certified}.
\vspace{-0.1in}
\section{Preliminaries}
\vspace{-0.1in}
In this section, we first introduce several basic definitions and notations. In general, we denote any randomized mechanism by $\mathcal{M}(\cdot)$, which outputs a random variable depending on the input. We represent any deterministic classifier that outputs a prediction label by $f(\cdot)$. A commonly-used randomized classifier can be constructed by $g(\cdot) = f(\mathcal{M}(\cdot))$. We denote a data sample and its label by $\xb$ and $y$, respectively. An $\ell_p$-norm ball centered at $\xb$ with radius $r$ is represented by $\mathbb{B}_p(\xb, r)$. We say a data sample $\xb'$ is in the $\mathbb{B}_p(\xb, r)$ iff $\|\xb' - \xb\|_p \leq r$. Next, we can detail the frameworks in \cite{lecuyer2018certified} and \cite{li2019certified}, {\em i.e.,} PixelDP and R\'{e}nyi Divergence-based Bound. 
\vspace{-0.06in}
\paragraph{PixelDP} To the best of our knowledge, PixelDP \cite{lecuyer2018certified} is the first framework to prove that randomized smoothing is a certified defense by connecting the concepts of adversarial robustness and differential privacy. The definition of adversarial robustness in the framework of PixelDP can be stated as follows:
\begin{definition}[PixelDP \cite{lecuyer2018certified}]\label{def:pixeldp}
For any $\xb \in \mathbb{R}^d$ and $\forall \xb' \in \mathbb{B}_p(\xb, r)$,
if a randomized mechanism $\mathcal{M}(\cdot)$ satisfies 
\begin{align}
    \forall S \subseteq \mathcal{O}, 
    P(\mathcal{M}(\xb)\in S) \leq e^\epsilon P(\mathcal{M}(\xb')\in S) + \delta,
\end{align}
where $\mathcal{O}$ denotes the output space of $\mathcal{M}(\cdot)$.
Then we can say $\mathcal{M}(\cdot)$ is
$(\epsilon, \delta)$-PixelDP (in $\mathbb{B}_p(\xb, r)$).
\end{definition}
\cite{lecuyer2018certified} connects PixelDP with adversarial robustness by the following lemma.
\begin{lemma}[Robustness Condition \cite{lecuyer2018certified}]\label{lemma:pixeldp}
Suppose $\Gb(\cdot)$ is randomized K-class classifier defined by $\Gb(\xb) = (\Gb_1(\xb),...,\Gb_K(\xb))$ that satisfies $(\epsilon, \delta)$-PixelDP (in $\mathbb{B}_p(\xb, r)$). For class $k$, if
\begin{align}
    \mathbb{E}[\Gb_k(\xb)] > e^{2\epsilon}\max_{i:i\neq k}\mathbb{E}[\Gb_i(\xb)] + (1 + e^\epsilon)\delta,
\end{align} 
then the classification result $(\mathbb{E}[\Gb_1(\xb)],...,\mathbb{E}[\Gb_K(\xb)])$ is robust\footref{foot:robustness} in $\mathbb{B}_p(\xb, r)$, {\em i.e.,} $\forall \xb'\in \mathbb{B}_p(\xb, r)$, $\argmax_i \mathbb{E}[\Gb_i(\xb')] = k$.
\end{lemma}
Note that the definition of the randomized classifier $\Gb(\cdot)$ is different from the definition of $g(\cdot)$ since the output of $g(\cdot)$ is a scalar not a vector (prediction label). $g(\cdot)$ is more popular in the follow-up works such as \cite{li2019certified, cohen2019certified}. 
\cite{lecuyer2018certified} mainly utilizes two mechanisms, {\em i.e.,} the Laplace mechanism and Gaussian mechanism, to guarantee PixelDP. Specifically, adding Laplace noise ({\em i.e.,} $p(z) = \frac{1}{2b}\exp{(-\frac{|z|}{b})}$) to the data samples can certify $(\epsilon, 0)$-PixelDP in $\mathbb{B}_1(\xb,  b\epsilon)$ for any $\xb$, and adding Gaussian noise ({\em i.e.,} $p(z) = \frac{1}{\sqrt{2\pi}\sigma}\exp{(-\frac{z^2}{2\sigma^2})}$) can certify $(\epsilon, \delta)$-PixelDP in $\mathbb{B}_2(\xb,  \frac{\sigma\epsilon}{\sqrt{2\log{1.25/\delta}}})$ for any $\xb$.
\vspace{-0.06in}
\paragraph{R\'{e}nyi Divergence-based Bound}
\cite{li2019certified} proves a tighter estimate (compared with \cite{lecuyer2018certified}) on the robust radii based on the following lemma. 
\begin{lemma}[R\'{e}nyi Divergence Lemma \cite{li2019certified}]\label{lemma:renyi}
Let $P=(p_1, p_2, ..., p_k)$ and $Q=(q_1, q_2, ..., q_k)$ be two multinomial distributions. If the indices of the largest probabilities \textbf{do not} match on $P$ and $Q$, then the R\'{e}nyi divergence between $P$ and $Q$, {\em i.e.,} $D_\alpha(P||Q)$\footnote{For $\alpha\in (1,\infty)$, $D_\alpha(P||Q)$ is defined as $D_\alpha(P||Q) = \frac{1}{\alpha-1} \log \mathbb{E}_{x\sim Q}(\frac{P(x)}{Q(x)})^\alpha$.}, satisfies
\begin{align}
       D_\alpha(P||Q) \geq -\log(1 - p_{(1)} - p_{(2)} + 2(\frac{1}{2}(p_{(1)}^{1-\alpha} + p_{(2)}^{1-\alpha}))^{\frac{1}{1-\alpha}}).
    \nonumber 
\end{align}
where $p_{(1)}$ and $p_{(2)}$ refer to the largest and the second largest probabilities in $\{p_i\}$, respectively.
\end{lemma}
If the Gaussian mechanism is applied to certify $\ell_2$-norm robustness, 
then we have the following bound of the robust radii. 
\begin{lemma}[\cite{li2019certified}]\label{lemma:renyi_gauss}
Let $f$ be any deterministic classifier and  $g(\xb)=f(\xb+\zb)$ be its corresponding randomized classifier for sample $\xb\in \mathbb{R}^d$, where $\zb \sim \mathcal{N}(0, \sigma^2I_d)$. Then $\forall \xb'\in \mathbb{B}_p(\xb, r)$, $\argmax_{y} P(g(\xb) = y) = \argmax_{y'} P(g(\xb') = y')$, i.e., $g(\cdot)$ is robust in $\mathbb{B}_p(\xb, r)$, and
the $\ell_2$ robust radii $r$ that could be certified is given by
\begin{align}
       r^2 \leq \sup_{\alpha > 1}-\frac{2\sigma^2}{\alpha}\log(&1 - p_{(1)} - p_{(2)} +  2(\frac{1}{2}(p_{(1)}^{1-\alpha} + p_{(2)}^{1-\alpha}))^{\frac{1}{1-\alpha}}).  
\end{align}
\end{lemma}
$p_{(1)}$ and $p_{(2)}$ refer to the largest and the second largest probabilities in $\{p_i\}$, {\em where $p_i$ is the probability that $g(\xb)$ returns the $i$-th class, {\em i.e.,} $p_i = P(g(\xb) = i)$.}
\vspace{-0.1in}
\section{ Framework Overview}\label{sec:over}
\vspace{-0.1in}
In this section, we present a generic framework based on the Definition \ref{def:max_renyi}, \ref{def:DMR}, and \ref{def:metric}, for assessing randomized mechanisms. According to Definition~\ref{def:DMR}, our framework applies to different norms. Moreover, we show that our proposed framework connects the existing general frameworks in \cite{lecuyer2018certified, li2019certified} by Theorem~\ref{thm:connect_le} \& \ref{thm:connect_li}. {\em Also, we note that it is difficult to involve the framework in \cite{cohen2019certified} since \cite{cohen2019certified} 
restricts the additive noise of the randomized mechanism to be isotropic such as Gaussian noise, while in our framework, we do not need to specify the type of additive noise.}
\vspace{-0.1in}
\subsection{Main Definitions}
\vspace{-0.1in}
Under our framework, the definition of adversarial robustness is induced by maximal relative R\'{e}nyi divergence (MR divergence), namely $D_{MR}$ robustness, so we start from introducing the definition of MR divergence.
\begin{definition}[Maximal Relative R\'{e}nyi Divergence]\label{def:max_renyi}
	The Maximal Relative R\'{e}nyi Divergence $D_{MR}(P \| Q)$ of distributions $P$ and $Q$ is defined as
	\begin{equation}
		D_{MR}(P\|Q)=\max_{\alpha\in(1, \infty) }\frac{D_\alpha (P\|Q)}{\alpha},
	\end{equation}
\end{definition}
where $D_\alpha (P\|Q)$ is the R\'{e}nyi Divergence between $P$ and $Q$. Using $D_{MR}$ as the probability measure, we can define adversarial robustness as follows:
\begin{definition}[$D_{MR}$ Robustness]\label{def:DMR}
	We say a randomized (smoothing) mechanism $\mathcal{M}(\cdot)$  
	is $(r, D_{MR}, \|\cdot\|_p, \epsilon)$-robust if for any $\xb \in \mathbb{R}^d$ and $\forall \xb'\in \mathbb{B}_p(\xb, r)$,
	\begin{align}\label{eq:DMR}
	D_{MR}&(\mathcal{M}(\xb)\|\mathcal{M}(\xb')) \leq \epsilon.
	\end{align}
	If a randomized smoothing classifier $g(\cdot)$ satisfies the above condition,
	we can say it is a $(r, D_{MR}, \|\cdot\|_p, \epsilon)$-robust classifier or 
	it certifies $(r, D_{MR}, \|\cdot\|_p, \epsilon)$-robustness.
\end{definition}
A property of $D_{MR}$ robustness we use throughout this paper is its postprocessing property, which can be stated as follows:
\begin{corollary}[Postprocessing Property]\label{thm:post}
	Let 
	$g(\xb) = f(\mathcal{M}(\xb))$ be a randomized classifier, where $f(\cdot)$ is any deterministic function (classifier). $g(\cdot)$ is $(r, D_{MR}, \|\cdot\|_p, \epsilon)$-robust if $\mathcal{M}(\cdot)$ is $(r, D_{MR}, \|\cdot\|_p, \epsilon)$-robust.
\end{corollary}
This postprocessing property can be easily proved by $D_\alpha(f(\mathcal{M}(\xb))\|f(\mathcal{M}(\xb'))) \leq D_\alpha(\mathcal{M}(\xb)\|\mathcal{M}(\xb'))$  for any $\alpha\in (1, \infty)$ \cite{van2014renyi}. This property allows us to only concentrate on the randomized mechanism $\mathcal{M}(\cdot)$ without considering the specific form of the deterministic classifier $f(\cdot)$, and therefore makes the framework applicable to an arbitrary neural network.
\vspace{-0.1in}
\subsection{Connections between $D_{MR}$ robustness and the existing frameworks}
\vspace{-0.1in}
The framework defined by Definition~\ref{def:max_renyi} \& \ref{def:DMR} is generic since it is closely connected with the existing ones \cite{lecuyer2018certified, li2019certified}. Here we demonstrate the connections by the following two theorems. 
\begin{theorem}[$D_{MR}$ Robustness \& PixelDP]\label{thm:connect_le}
If $\mathcal{M}(\cdot): \mathbb{R}^d \mapsto \mathbb{R}^d$ is $(r, D_{MR}, \|\cdot\|_p, \epsilon)$-robust, then $\mathcal{M}(\cdot)$ is also $(\epsilon + 2\sqrt{\log{(1/\delta)}\epsilon}, \delta)$-PixelDP in $\mathbb{B}_p(\xb, r)$ for any $\xb \in \mathbb{R}^d$. 
\end{theorem}
We note that the opposite of Theorem~\ref{thm:connect_le} holds only when $\delta = 0$, which indicates our framework is a relaxed version of the PixelDP framework. But this should not be a surprise since most of the following frameworks \cite{li2019certified, cohen2019certified, dvijotham2020framework} can somehow be considered more relaxed than the PixelDP framework and thus yield tighter certified bounds.
Similarly, our framework can provide the same bound on the robust radius as in \cite{li2019certified}, which is tighter than the bound in \cite{lecuyer2018certified} (Theorem~\ref{thm:connect_li}).
\begin{theorem}[$D_{MR}$ Robustness \& R\'{e}nyi Divergence-based Bound]\label{thm:connect_li}
If a randomized classifier $g(\cdot)$ is $(r, D_{MR}, \|\cdot\|_p, \epsilon)$-robust, then we have $\forall \xb'\in \mathbb{B}_p(\xb, r)$, $\argmax_{y} P(g(\xb) = y) = \argmax_{y'} P(g(\xb') = y')$ as long as 
\begin{align} 
\epsilon \leq \sup_{\alpha > 1} -\frac{1}{\alpha}\log(1 - p_{(1)} - p_{(2)} + 2(\frac{1}{2}(p_{(1)}^{1-\alpha} + p_{(2)}^{1-\alpha}))^{\frac{1}{1-\alpha}}),
\end{align}
\end{theorem}
where $p_{(1)}$ and $p_{(2)}$ also refer to the largest and the second largest probabilities in $\{p_i\}$, and $p_i$ is the probability that $g(\xb)$ returns the $i$-th class, {\em i.e.,} $p_i = P(g(\xb) = i)$.
Based on the above theorem, we can derive the same $\ell_2$ robust radius as in Lemma~\ref{lemma:renyi_gauss} \cite{li2019certified}. We will detail how to derive the $\ell_2$ robust radius after Theorem~\ref{thm:l2_gauss} in Section~\ref{sec:l2}. 

An interpretation of Theorem~\ref{thm:connect_le} and \ref{thm:connect_li} is that, as long as we can use a randomized mechanism with a certain amount of noise to certify $D_{MR}$ robustness, we can use the same mechanism with the same amount of noise to certify PixelDP and the R\'{e}nyi Divergence-based bound. Thus, Theorem~\ref{thm:connect_le} and \ref{thm:connect_li} indicate the assessment results based on the metric defined in Section~\ref{subsec:assessment} is very likely to generalize to the other frameworks.
\vspace{-0.1in}
\subsection{Assessment of Randomized Mechanisms}\label{subsec:assessment}
\vspace{-0.1in}
Since there are infinite randomized mechanisms, a natural problem is to determine whether a certain randomized mechanism is an appropriate option to certify adversarial robustness. However, we note that all the previous work \cite{li2019certified, cohen2019certified, salman2019provably} overlook this problem and assume the Gaussian mechanism to be an appropriate mechanism for certifying $\ell_2$-norm robustness without sufficient assessment. While in this paper, we attempt to provide a solution to this problem under our proposed framework. Specifically, we define a metric to assess randomized mechanisms as follows:

\begin{definition}\label{def:metric}
    \vspace{-0.03in}
    Specify a $p$-norm, a robust radius $r$, and an epsilon $\epsilon$, the magnitude (expected $\ell_\infty$-norm) of the additive noise required by a randomized mechanism $\mathcal{M}(\xb) = \xb + \zb$ to certify $(r, D_{MR}, \|\cdot\|_p, \epsilon)$-robustness is defined as the metric to assess the appropriateness of $\mathcal{M}(\cdot)$. 
    \vspace{-0.03in}
\end{definition}

We define this metric for assessing randomized mechanisms because {\em the accuracy of neural networks tends to decrease as the magnitude of the noise added to the inputs increases.} Note that if the magnitude of the noise required by a randomized classifier is too large, the accuracy of its predictions on clean samples will be very low, then robustness will be useless\footnote{Certified robustness only guarantees the predictions of the perturbed samples and the predictions of their clean samples are the same.\label{foot:robustness}}.
Given the above metric, we also need criteria to assess the (relative) appropriateness of a randomized mechanism. In this paper, we employ the lower bounds on the magnitude of the noise required by any randomized mechanism to certify $(r, D_{MR}, \|\cdot\|_p, \epsilon)$-robustness as the criteria. {\em We consider a randomized mechanism as an appropriate option if the gap between the magnitude of the additive noise required by this mechanism and the corresponding lower bound is small.} 
In the following, we will provide the lower bounds for $\ell_2$-norm and $\ell_\infty$-norm, {\em i.e., the two most popular norms}, and assess the appropriateness of the Gaussian and Exponential mechanisms for certifying $\ell_2$-norm and $\ell_\infty$-norm robustness. In Appendix, we generalize our framework to $\ell_p$-norm for any $p\geq2$.
\vspace{-0.1in}
\section{Assessing Mechanisms for Certifying $\ell_2$-norm Robustness}\label{sec:l2}
\vspace{-0.1in}
In this section, we first elaborate on how the Gaussian mechanism certifies $D_{MR}$ robustness, and then provide the lower bound on the magnitude of the additive noise required by any randomized mechanism ($\mathcal{M}(\xb) = \xb + \zb$) to certify $\ell_2$-norm robustness. By comparing the magnitude of the additive noise required by the Gaussian mechanism with the lower bound, we conclude that the Gaussian mechanism is an appropriate option to certify $\ell_2$-norm robustness.
\vspace{-0.05in}
\begin{theorem}[Gaussian Mechanism for Certifying $\ell_2$-norm robustness]\label{thm:l2_gauss}
	Let $f$ be any deterministic classifier and $g(\xb)=f(\mathcal{M}(\xb))$ be its corresponding randomized classifier for sample $\xb \in \mathbb{R}^d$, where $\mathcal{M}(\xb) = \xb + \zb$ with $\zb \sim \mathcal{N}(0, \sigma^2I_d)$. Then, $g(\cdot)$ is $(r, D_{MR}, \|\cdot\|_2, \frac{r^2}{2\sigma^2})$-robust. 
\end{theorem}
According to Theorem~\ref{thm:connect_li}, if we substitute $\epsilon$ with $\frac{r^2}{2\sigma^2}$, $r$ can be given by 
$
 r^2 \leq \sup_{\alpha > 1} -\frac{2\sigma^2}{\alpha} \log{(1 - p_{(1)} - p_{(2)} + 2(\frac{1}{2}(p_{(1)}^{1-\alpha} +  p_{(2)}^{1-\alpha}))^{\frac{1}{1-\alpha}})},   
$
which is same as the bound of the robust radii in \cite{li2019certified} (Lemma~\ref{lemma:renyi_gauss}).   
To provide a criterion for the assessment of randomized mechanisms in the $\ell_2$-norm case,
we prove a lower bound on the magnitude of the additive noise $\zb$ required by any randomized mechanism $\mathcal{M}(\xb)=\xb+\zb$ to ensure that $\mathcal{M}(\xb)$ (as well as $f(\mathcal{M}(\xb))$) is $(r, D_{MR}, \|\cdot\|_2, \epsilon)$-robust. As mentioned in Section~\ref{subsec:assessment}, if the magnitude of the additive Gaussian noise is close to the lower bound, then the Gaussian mechanism is considered as an appropriate option. The lower bound is given by the following theorem.
\vspace{-0.05in}
\begin{theorem}[$\ell_2$-norm Criterion for Assessment]\label{thm:l2lower}
	For any $\epsilon\leq O(1)$, if there is an $(r, D_{MR}, \|\cdot\|_2, \epsilon)$-robust randomized mechanism $\mathcal{M}(\xb)=\xb+\zb: \mathbb{R}^d \mapsto \mathbb{R}^d$ that satisfies 
	\begin{equation}
		\mathbb{E}[\|\zb\|_\infty]= \mathbb{E}[\|\mathcal{M}(\xb)-\xb\|_\infty] \leq \alpha
	\end{equation}
	for some $\alpha \leq O(1)$, then it must be true that  
	$\alpha \geq \Omega(\frac{r}{\sqrt{\epsilon}})$. 
	In another word, $\Omega(\frac{r}{\sqrt{\epsilon}})$ is the lower bound of the (expected) magnitude of the additive random noise.
\end{theorem}

Note that proving this theorem on $\mathbb{R}^d$ is non-trivial, which is detailed in Appendix.
Theorem \ref{thm:l2lower} indicates that the magnitude (expected $\ell_\infty$-norm) of the additive noise should be at least $\Omega(\frac{r}{\sqrt{\epsilon}})$ to certify $(r, D_{MR}, \|\cdot\|_2, \epsilon)$-robustness. For the Gaussian mechanism, the expected $\ell_\infty$-norm of the additive noise is $O(\sigma\sqrt{\log d})$ according to \cite{orabona2015optimal}, which is $O(\frac{r}{\sqrt{\epsilon}}\sqrt{\log d})$ to guarantee $(r, D_{MR}, \|\cdot\|_2, \epsilon)$-robustness, according to Theorem~\ref{thm:l2_gauss}.
This means the gap between the magnitude of the noise required by the Gaussian mechanism and the lower bound is bounded by $O(\sqrt{\log d})$. 
\begin{remark}\label{remark:scale}
    We say Gaussian mechanism is an appropriate option because the gap $O(\sqrt{\log d})$ is small for most commonly-used datasets. For instance, for CIFAR-10 ($d=3072$), $\sqrt{\log_e d} \approx 2.83$, and for ImageNet ($d=150528$), $\sqrt{\log_e d} \approx 3.45$.
\end{remark}
Equivalently, if we fix the expected $\ell_\infty$-norm of the additive noise as $\alpha$, the radius $r$ that can be certified by any $(r, D_{MR}, \|\cdot\|_2, \epsilon)$-robust randomized mechanism is upper bounded by $O(\alpha \sqrt{\epsilon})$, according to Theorem~\ref{thm:l2lower}. For the Gaussian mechanism, since $\alpha = O(\sigma\sqrt{\log d})$, the certified robust radius $r$ is $O(\frac{\alpha \sqrt{\epsilon}}{\sqrt{\log d}})$\footnote{The theoretical results of the scales of the robust radii are verified by experiments.\label{foot:verification}}, according to Theorem~\ref{thm:l2_gauss}. 
This means the gap between the upper bound of the robust radius and the radius certified by the Gaussian mechanism is also $O(\sqrt{\log d})$.
\vspace{-0.1in}
\section{Assessing Mechanisms for Certifying $\ell_\infty$-norm Robustness}\label{sec:linfty}
\vspace{-0.1in}
In this section, we first discuss the possibility of using the Exponential mechanism, an analogue of the Gaussian mechanism in the $\ell_\infty$-norm case, to certify $\ell_\infty$-norm robustness. Then, we prove the lower bound on the magnitude of the additive noise required by any randomized mechanism to certify $\ell_\infty$-norm robustness. By comparing the magnitude of the noise required by the Exponential mechanism with the lower bound, we conclude that the Exponential mechanism is not an appropriate option to certify $\ell_\infty$-norm robustness. Surprisingly, we find that the Gaussian mechanism is a more appropriate option than the Exponential mechanism to certify $\ell_\infty$-norm robustness.

We first recall the form of the density function of Gaussian noise:
$ p(\zb)\propto \exp(-\frac{\|\zb\|_2^2}{\sigma^2})$.
Based on this, we conjecture that, to certify $\ell_\infty$-norm robustness, we can sample the noise using the Exponential mechanism, an analogue of the Gaussian mechanism in the $\ell_\infty$-norm case: 
\begin{align}\label{eq:exponential}
	p(\zb) \propto \exp{(-\frac{\|\zb\|_\infty}{\sigma})}.
\end{align}
We show in the following theorem that randomized smoothing using the Exponential mechanism can certify $(r, D_{MR}, \|\cdot\|_\infty, \frac{r^2}{2\sigma^2})$-robustness, which is seemingly an extension of the $\ell_2$-norm case. However, its required magnitude of noise is $O(d)$, which implies it is unscalable to high-dimensional data, {\em i.e.,} The Exponential mechanism should not be an appropriate mechanism to certify $\ell_\infty$-norm robustness. This conclusion is further verified by our assessment method, which will be detailed later. 
\begin{theorem}[Exponential Mechanism for Certifying $\ell_\infty$-norm Robustness]\label{thm:exponential}
Let $f$ be any deterministic classifier and  $g(\xb)=f(\mathcal{M}(\xb))$ be its corresponding randomized classifier for sample $\xb\in \mathbb{R}^d$, where $\mathcal{M}(\xb) = \xb + \zb$ with $\zb$ sampled from the Exponential mechanism. Then, $g(\cdot)$ is $(r, D_{MR}, \|\cdot\|_\infty, \frac{r}{\sigma})$-robust and also $(r, D_{MR}, \|\cdot\|_\infty, \frac{r^2}{2\sigma^2})$-robust.
\end{theorem}
According to Theorem~\ref{thm:connect_li}, if we substitute $\epsilon$ with $\frac{r}{\sigma}$ or $\frac{r^2}{2\sigma^2}$, then $r$ can be given by 
$
 r \leq \sup_{\alpha > 1} -\frac{\sigma}{\alpha} \log{(1 - p_{(1)} - p_{(2)} + 2(\frac{1}{2}(p_{(1)}^{1-\alpha} +  p_{(2)}^{1-\alpha}))^{\frac{1}{1-\alpha}})},   
$
or
$
 r^2 \leq \sup_{\alpha > 1} -\frac{2\sigma^2}{\alpha} \log(1 - p_{(1)} - p_{(2)} + 2(\frac{1}{2}(p_{(1)}^{1-\alpha} +  p_{(2)}^{1-\alpha}))^{\frac{1}{1-\alpha}}).  
$
Comparing this result and Theorem~\ref{thm:l2_gauss}, we can see that 
randomized smoothing via the Exponential mechanism certifies a similar form of the radius as that certified by the Gaussian mechanism in the $\ell_2$-norm case, indicating similarity in their robustness guarantees. However, the following corollary shows that the magnitude of the noise required by the Exponential mechanism is much larger than that of the Gaussian mechanism in the $\ell_2$-norm case. 
\begin{corollary}\label{thm:exponential_noise_bound}
	For the Exponential mechanism that can guarantee Theorem~\ref{thm:exponential},
		$\mathbb{E}[\|\zb\|_\infty]=d\sigma. $
\end{corollary}
Equivalently, if we fix the expected $\ell_\infty$-norm of the additive noise as $\alpha$, according to Theorem~\ref{thm:exponential} \& Corollary~\ref{thm:exponential_noise_bound}, the robust radius $r$ certified by the Exponential mechanism is $\max\{O(\frac{\alpha\epsilon}{d}), O(\frac{\alpha\sqrt{\epsilon}}{d})\}$\footref{foot:verification}. The following theorem shows that there is a huge gap between the additive noise required by the Exponential mechanism and the lower bound, indicating that the Exponential mechanism is indeed not an appropriate option for certifying $\ell_\infty$-norm robustness here.

\begin{theorem}[$\ell_\infty$-norm Criterion for Assessment]\label{thm:linflower}
For any $\epsilon\leq O(1)$, if there is an $(r, D_{MR}, \|\cdot\|_\infty, \epsilon)$-robust mechanism $\mathcal{M}(\xb)=\xb+\zb: \mathbb{R}^d \mapsto \mathbb{R}^d$ that satisfies
	\begin{equation}
		\mathbb{E}[\|\zb\|_\infty]= \mathbb{E}[\|\mathcal{M}(\xb)-\xb\|_\infty] \leq \alpha
	\end{equation}
	for some $\alpha \leq O(1)$, then it must be true that $\alpha \geq \Omega(\frac{r\sqrt{d}}{\sqrt{\epsilon}})$. In another word, $\Omega(\frac{r\sqrt{d}}{\sqrt{\epsilon}})$ is the lower bound of the (expected) magnitude of the additive random noise. 
\end{theorem}
According to Corollary \ref{thm:exponential_noise_bound} and Theorem \ref{thm:exponential}, for the Exponential mechanism, its required magnitude of noise is $O(\frac{rd}{\sqrt{\epsilon}})$ or $O(\frac{rd}{\epsilon})$ to certify $(r, D_{MR}, \|\cdot\|_\infty, \epsilon)$-robustness. Compared with Theorem~\ref{thm:linflower}, we can see that the gap between the magnitude of the noise required by the Exponential mechanism and the lower bound is $O(\sqrt{d})$, which can be very large for high-dimensional datasets. Therefore, we can conclude that the Exponential mechanism is probably not an appropriate mechanism for certifying $\ell_\infty$-norm robustness.
Surprisingly, the following theorem shows that the Gaussian mechanism is an appropriate choice for certifying $(r, D_{MR}, \|\cdot\|_\infty, \epsilon)$-robustness.
\vspace{-0.02in}
\begin{theorem}[Gaussian Mechanism for Certifying $\ell_\infty$-norm robustness]\label{thm:gauss_linf}
	Let $r, \epsilon>0$ be some fixed number and  $\mathcal{M}(\xb)=\xb+\zb$ with $\zb \sim \mathcal{N}(0, \frac{dr^2}{2\epsilon}I_d)$. Then, $\mathcal{M}(\cdot)$  is $(r, D_{MR}, \|\cdot\|_\infty, \epsilon)$-robust, and $\mathbb{E}[\|\zb\|_\infty]= \mathbb{E}[\|\mathcal{M}(\xb)-\xb\|_\infty]$ is upper bounded by $O(\frac{r\sqrt{d\log d}}{\sqrt{\epsilon}})$.
\end{theorem}
\vspace{-0.05in}
From Theorem~\ref{thm:linflower} and \ref{thm:gauss_linf}, we can see that the gap between the magnitude of the noise required by the Gaussian mechanism and the lower bound is also $O(\sqrt{\log d})$. Thus, we can say the Gaussian mechanism is an appropriate option to certify $\ell_\infty$-norm robustness (see Remark~\ref{remark:scale}). Equivalently, if we fix the expected $\ell_\infty$-norm of the additive noise as $\alpha$, the certified robust radius is $O(\frac{\alpha\sqrt{\epsilon}}{\sqrt{d\log d}})$\footref{foot:verification}.
\vspace{-0.05in}
\begin{remark}
Note that in the previous sections, we only consider $\ell_2$-norm and $\ell_\infty$-norm and the corresponding mechanisms because they are the two most important norms. But actually, we can extend our framework to $\ell_p$-norm for any $p\geq 2$. See Section~\ref{sec:lp} in Appendix. 
\end{remark}

\vspace{-0.15in}
\section{Experiments}
\vspace{-0.1in}
\paragraph{Datasets and Models}
Our theories are verified on two widely-used datasets, {\em i.e.,} CIFAR10 and ImageNet\footnote{Pixel value range is $[0.0, 1.0]$}. We follow \cite{cohen2019certified, salman2019provably} to use a 110-layer residual network and a ResNet-50 as the base models for CIFAR10 and ImageNet. The certified accuracy for radius $R$ is defined as the fraction of the test set whose certified radii are larger than $R$, and predictions are correct. We note that the lower bounds (criteria) are not verifiable by experiments since we are still not sure if there exist any randomized mechanism that can achieve those lower bounds. {\em So in the following, we mainly verify the theoretical results regarding the Gaussian mechanism and the Exponential mechanism. We provide more details about the numerical method (to compute the robust radii) and more experimental results compared to the other frameworks in Appendix in the supplementary material.}
\paragraph{Empirical Results}
In the following, we verify our framework by comparing our theoretical results of the $\ell_2/\ell_\infty$ robust radii with the $\ell_2/\ell_\infty$ radii at which the Gaussian/Exponential mechanism can certify $40\sim60\%$ accuracy in the experiments. Note that in the previous literature, $40\sim60\%$ robust accuracy is considered as a reasonably good performance \cite{madry2017towards, cohen2019certified}. Besides, selecting another reasonable accuracy does not affect the verification results too much because what our theories characterize are the asymptotic behaviors rather than the exact values of the robust radii.

In Fig.~\ref{fig:l2_noise}, we demonstrate the results of the Gaussian mechanism for certifying $\ell_2$-norm robustness. The red dashed lines show that the Gaussian mechanism can certify $40\sim60\%$ accuracy at $\ell_2~\mbox{radius}~=0.34$ (CIFAR-10, $d=3072$) and $\ell_2~\mbox{radius}~=0.29$ (ImageNet, $d=150568$), {\em i.e.,} approximately $1/\sqrt{\log d}$. These results verify that the $\ell_2$ radius certified by the Gaussian mechanism is $O(\frac{\alpha \sqrt{\epsilon}}{\sqrt{\log d}})$\footnote{$\alpha \leq O(1)$, and $\epsilon \leq O(1)$ (equality can hold).}.
We also argue that, $O(\frac{\alpha \sqrt{\epsilon}}{\sqrt{\log d}})$ is the scale of the largest certified $\ell_2$ radius ({\em i.e.,} $\frac{\sigma}{2}(\Phi^{-1}(p_{(1)})-\Phi^{-1}(p_{(2)}))$ \cite{cohen2019certified}) in the previous literature since the $\ell_\infty$-norm of the Gaussian noise $\alpha$ is $O(\sigma\sqrt{\log d})$. This argument is verified by Fig.~\ref{fig:cohen_l2_gauss} \& \ref{fig:salman_l2_gauss} in Appendix.

Fig.~\ref{fig:linf_noise_gauss} (1\&3 subfigures) shows that the Gaussian mechanism certifies $40\sim60\%$ accuracy at $\ell_\infty~\mbox{radius}~=6e-3$ on CIFAR-10 ($d=3072$) and $\ell_\infty~\mbox{radius}~=1.1e-3$ on ImageNet ($d=150568$), {\em i.e.,} approximately $O(1/\sqrt{d\log d})$. These results verify that the $\ell_\infty$ radius certified by the Gaussian mechanism is $O(\frac{\alpha\sqrt{\epsilon}}{\sqrt{d\log d}})$. 
Fig.~\ref{fig:linf_noise_gauss} (2\&4 subfigures) also shows that the Exponential mechanism certifies approximately $40\sim60\%$ accuracy at $\ell_\infty~\mbox{radius}~=1.5e-4$ on CIFAR-10 and $\ell_\infty~\mbox{radius}~=7e-6$ on ImageNet, {\em i.e.,} approximately $O(1/d)$. These results verify that the $\ell_\infty$ robust radius certified by the Exponential mechanism scales in $O(\frac{\alpha\epsilon}{d})$ or $O(\frac{\alpha\sqrt{\epsilon}}{d})$.
If we compare the performance of the Gaussian mechanism and the Exponential mechanism in Fig.~\ref{fig:linf_noise_gauss}, we can see that the Gaussian mechanism is a much more appropriate option for certifying $\ell_\infty$-norm robustness. 
It is worth noting that the performance of the Gaussian mechanism can be better with the bound proved in \cite{cohen2019certified}, which is comparable to the other state-of-the-art approaches introduced in Section~\ref{sec:related}. 
We detail some results regarding the comparison in Appendix. 

\begin{figure}[h]
    \centering
	\includegraphics[width=0.24\columnwidth]{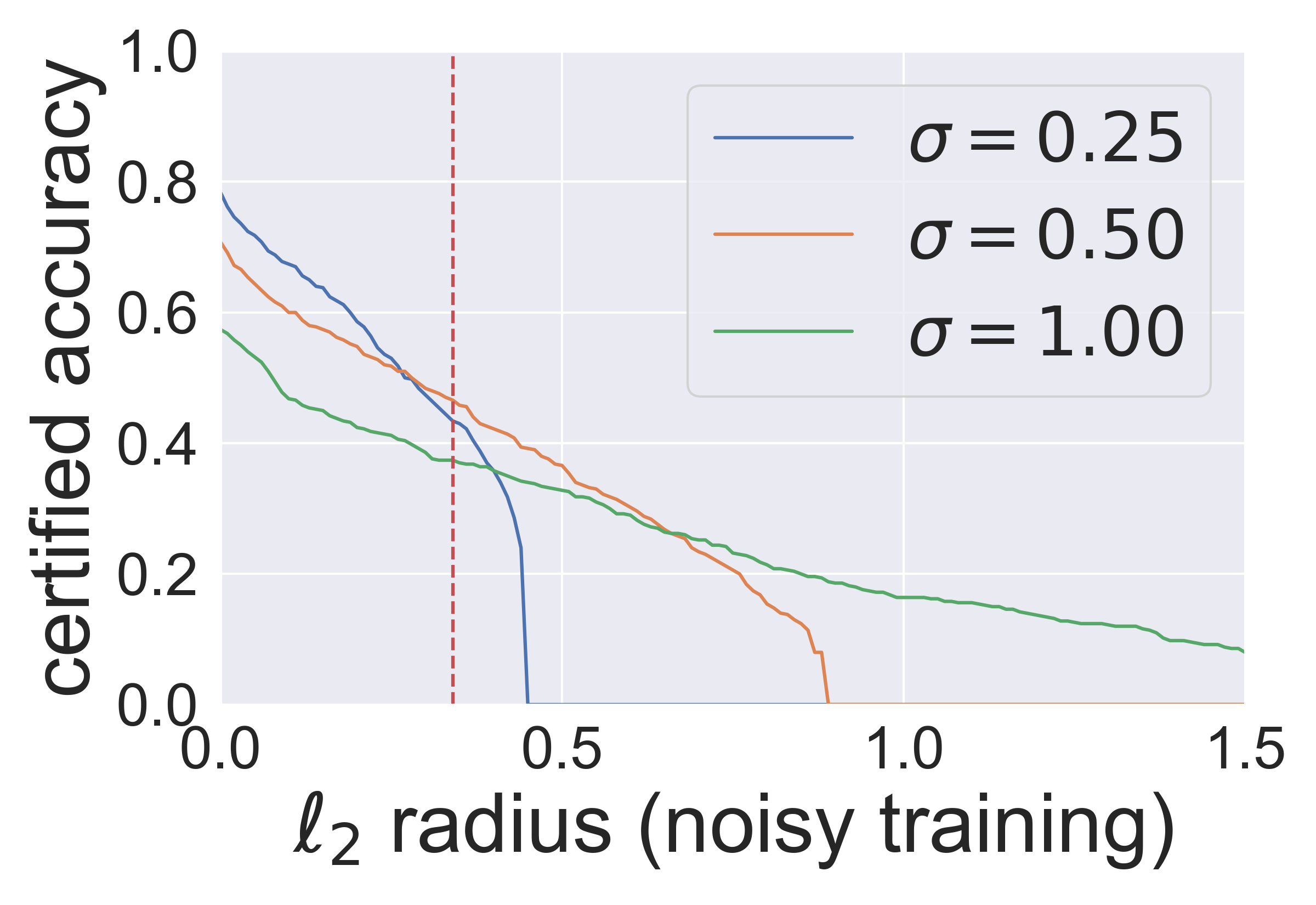}
	\centering
	\includegraphics[width=0.24\columnwidth]{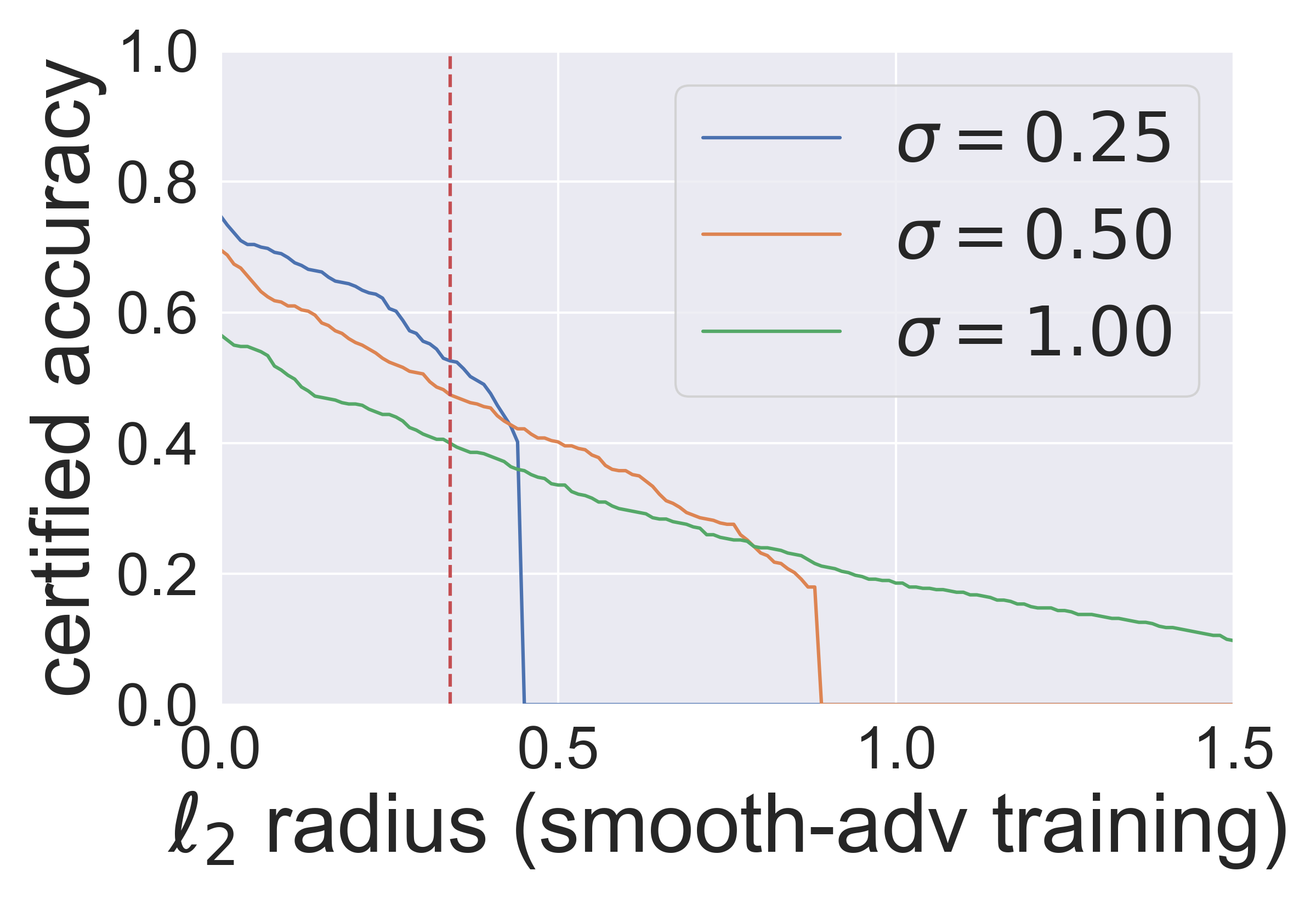}
	\includegraphics[width=0.24\columnwidth]{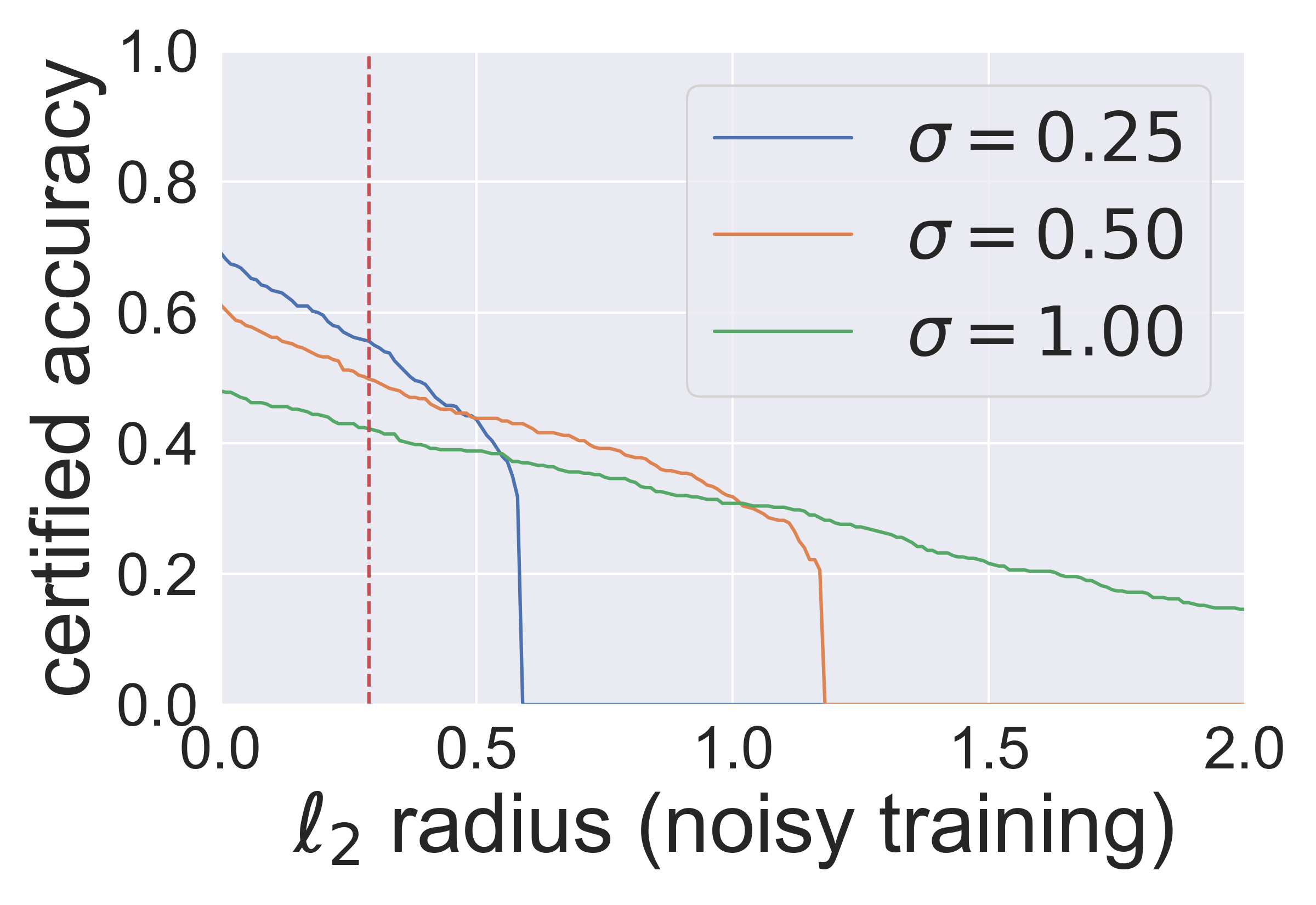}
	\centering
	\includegraphics[width=0.24\columnwidth]{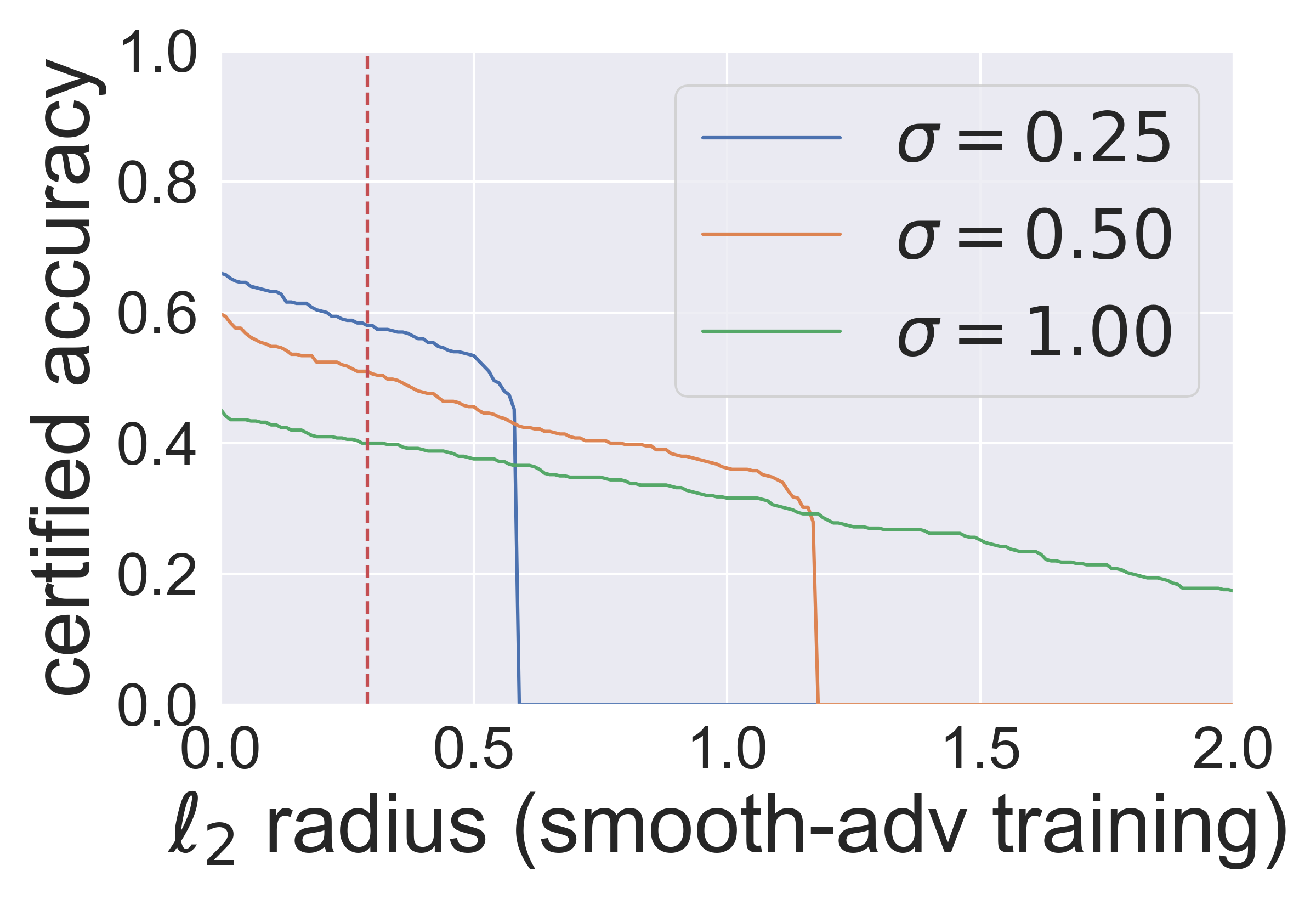}
	\vspace{-0.1cm}
	\caption{Certify $\ell_2$-norm robustness by the Gaussian mechanism: certified accuracy for CIFAR-10 (left two) and ImageNet (right two). Models: noisy training \cite{cohen2019certified} \& smooth-adv training \cite{salman2019provably}.}
	\vspace{-0.3cm}
	\label{fig:l2_noise}
\end{figure}
\begin{figure}[h]
    \centering
	\includegraphics[width=0.24\columnwidth]{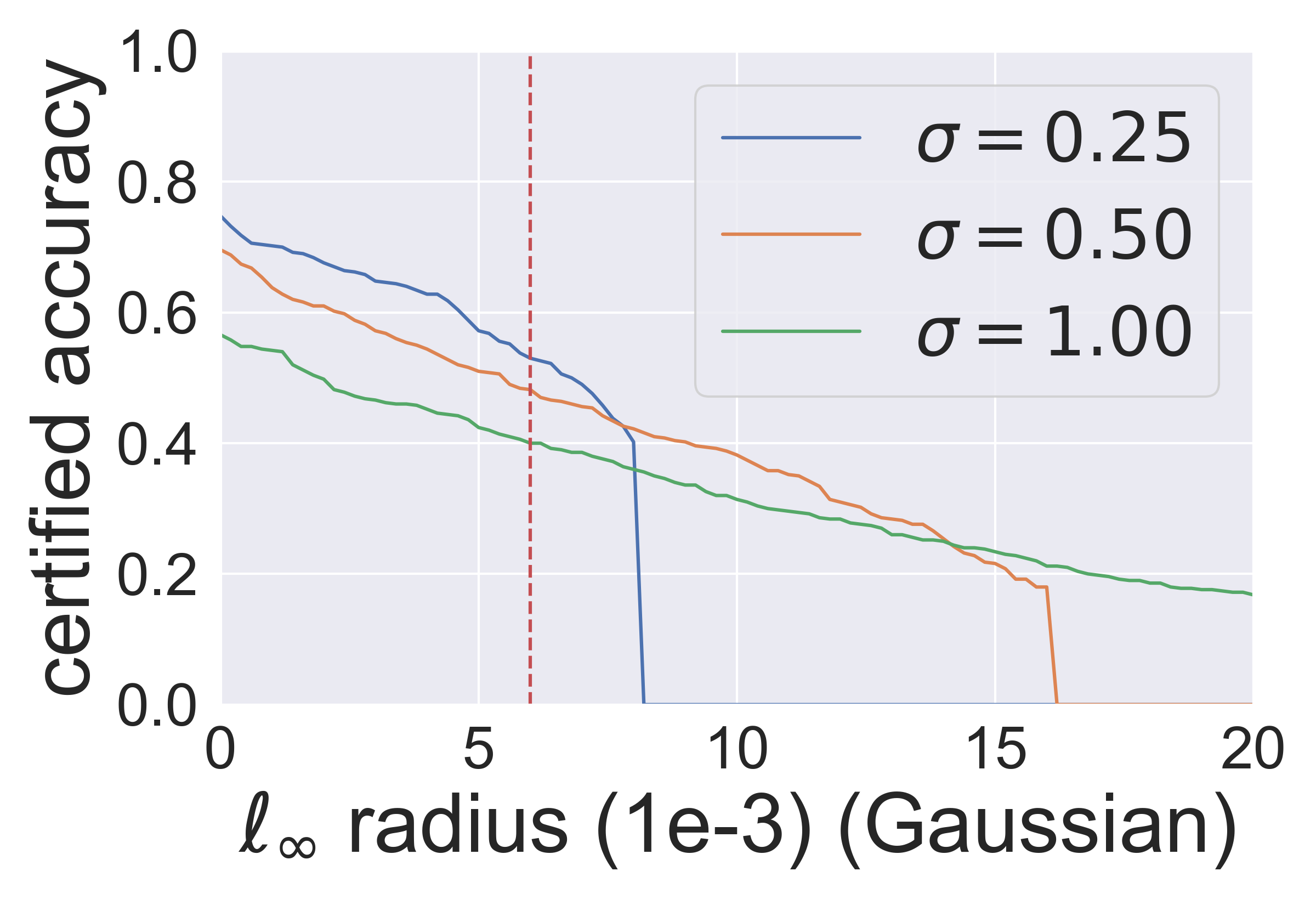}
	\includegraphics[width=0.24\columnwidth]{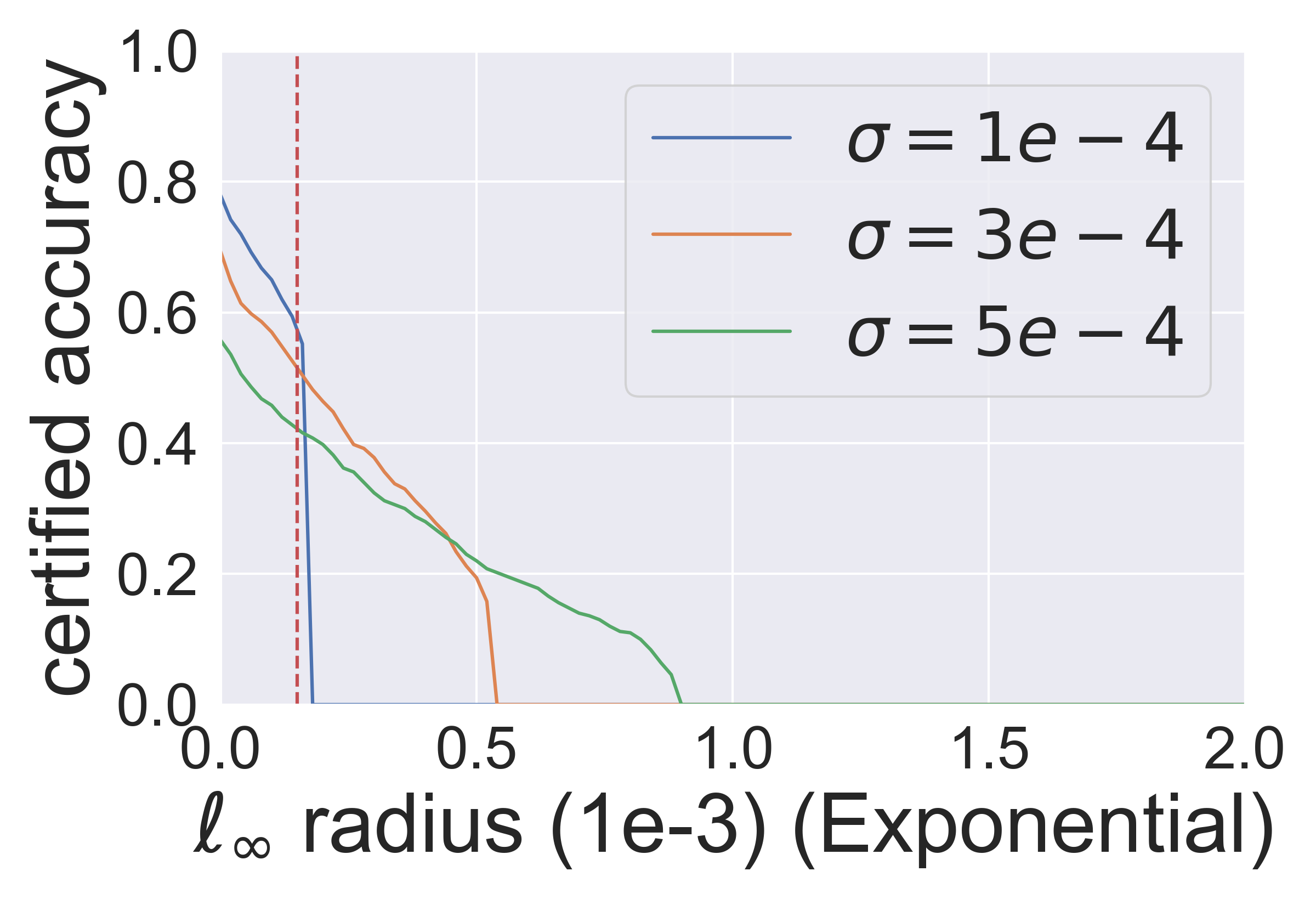}
	\includegraphics[width=0.24\columnwidth]{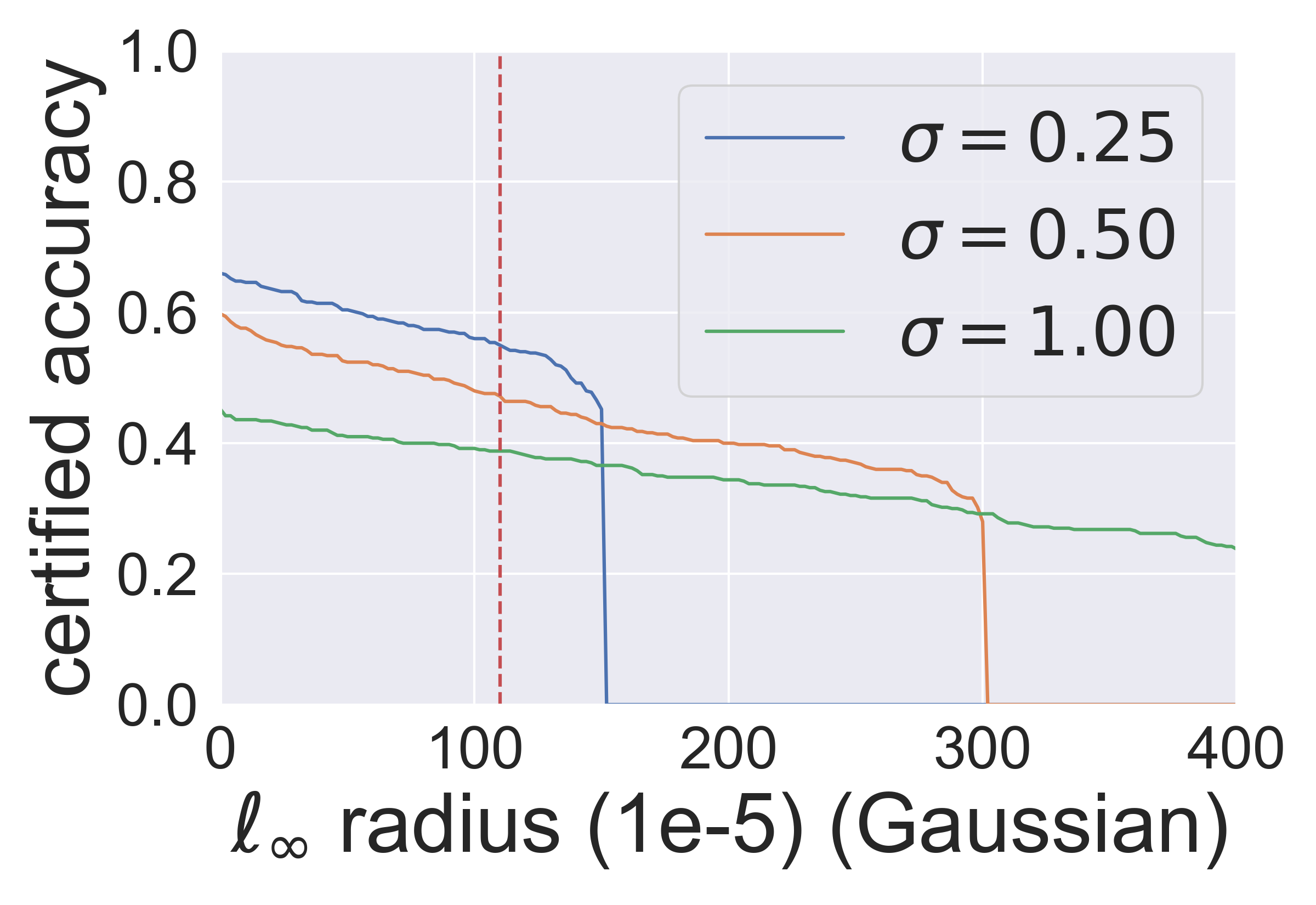}
	\includegraphics[width=0.24\columnwidth]{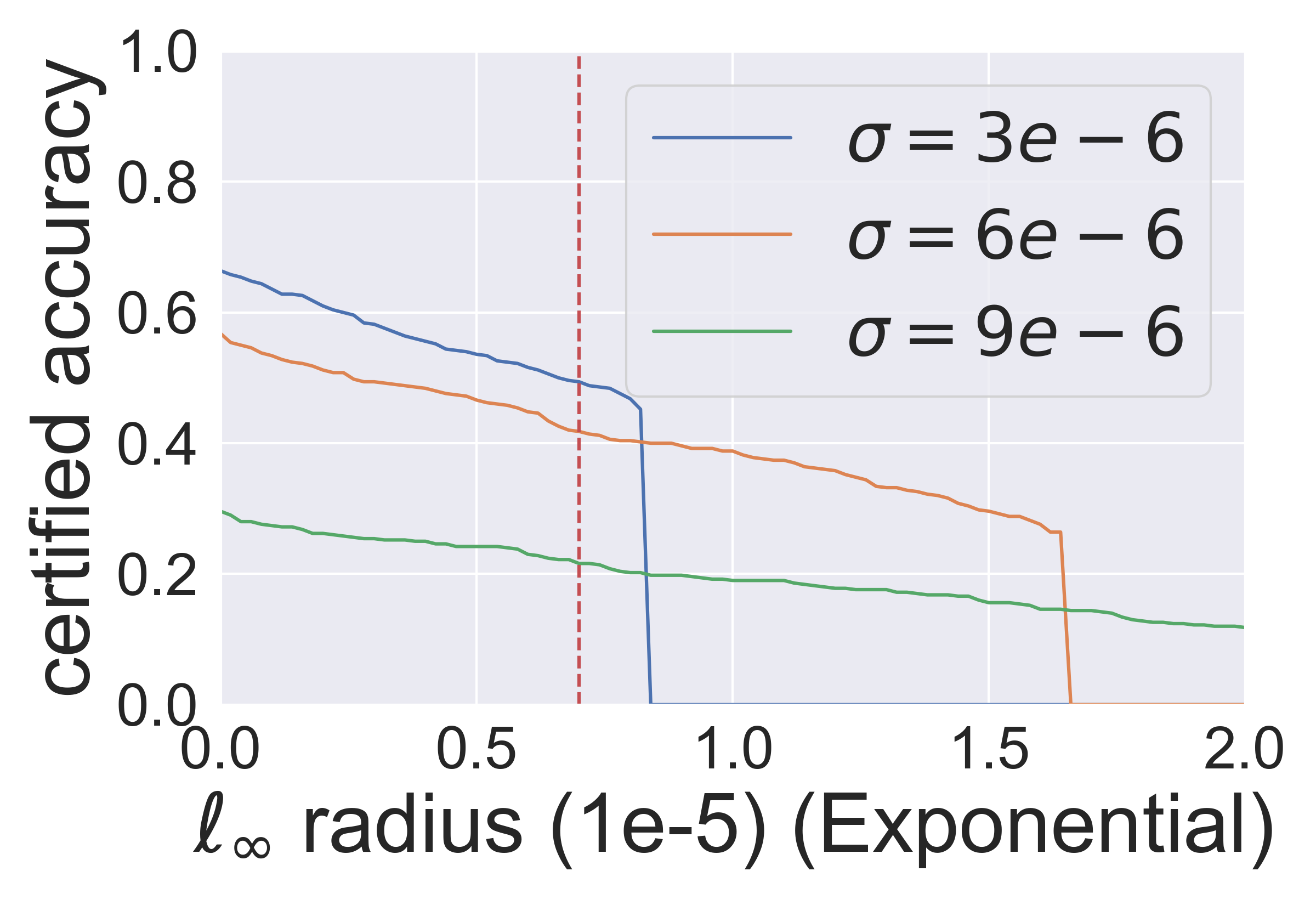}
	\vspace{-0.3cm}
	\caption{Certify $\ell_\infty$-norm robustness by the Gaussian/Exponential mechanism: certified accuracy for CIFAR-10 (left two) and ImageNet (right two). Model: smooth-adv training \cite{salman2019provably}.}
	\label{fig:linf_noise_gauss}
\end{figure}

\vspace{-0.15in}
\section{Conclusion}
\vspace{-0.1in}
In this paper, we present a generic and self-contained framework, which applies to different norms and connects the existing frameworks such as \cite{lecuyer2018certified, li2019certified}, for assessing randomized mechanisms.
Under our framework, we define the magnitude of the noise added by a randomized mechanism to certify a certain extent of robustness as the metric for assessing this mechanism. We also provide the general lower bounds on the magnitude of the additive noise as the assessment criteria. Comparing the noise required by the Gaussian and Exponential mechanism and the lower bounds, we conclude that (i) The Gaussian mechanism is an appropriate option to certify $\ell_2$-norm and $\ell_\infty$-norm robustness. (ii) The Exponential mechanism is not an appropriate mechanism to certify $\ell_\infty$-norm robustness. Moreover, we extend our assessment framework to $\ell_p$-norm for any $p\geq 2$.



\bibliographystyle{unsrt}
\bibliography{citation}

\begin{thebibliography}{10}

\bibitem{lecuyer2018certified}
Mathias Lecuyer, Vaggelis Atlidakis, Roxana Geambasu, Daniel Hsu, and Suman
  Jana.
\newblock Certified robustness to adversarial examples with differential
  privacy.
\newblock {\em arXiv preprint arXiv:1802.03471}, 2018.

\bibitem{li2019certified}
Bai Li, Changyou Chen, Wenlin Wang, and Lawrence Carin.
\newblock Certified adversarial robustness with additive noise.
\newblock In {\em Advances in Neural Information Processing Systems}, pages
  9459--9469, 2019.

\bibitem{krizhevsky2012imagenet}
Alex Krizhevsky, Ilya Sutskever, and Geoffrey~E Hinton.
\newblock Imagenet classification with deep convolutional neural networks.
\newblock In {\em Advances in neural information processing systems}, pages
  1097--1105, 2012.

\bibitem{cho2014learning}
Kyunghyun Cho, Bart Van~Merri{\"e}nboer, Caglar Gulcehre, Dzmitry Bahdanau,
  Fethi Bougares, Holger Schwenk, and Yoshua Bengio.
\newblock Learning phrase representations using rnn encoder-decoder for
  statistical machine translation.
\newblock {\em arXiv preprint arXiv:1406.1078}, 2014.

\bibitem{silver2016mastering}
David Silver, Aja Huang, Chris~J Maddison, Arthur Guez, Laurent Sifre, George
  Van Den~Driessche, Julian Schrittwieser, Ioannis Antonoglou, Veda
  Panneershelvam, Marc Lanctot, et~al.
\newblock Mastering the game of go with deep neural networks and tree search.
\newblock {\em nature}, 529(7587):484--489, 2016.

\bibitem{szegedy2013intriguing}
Christian Szegedy, Wojciech Zaremba, Ilya Sutskever, Joan Bruna, Dumitru Erhan,
  Ian Goodfellow, and Rob Fergus.
\newblock Intriguing properties of neural networks.
\newblock {\em arXiv preprint arXiv:1312.6199}, 2013.

\bibitem{goodfellow2014explaining}
Ian~J Goodfellow, Jonathon Shlens, and Christian Szegedy.
\newblock Explaining and harnessing adversarial examples.
\newblock {\em arXiv preprint arXiv:1412.6572}, 2014.

\bibitem{he2017adversarial}
Warren He, James Wei, Xinyun Chen, Nicholas Carlini, and Dawn Song.
\newblock Adversarial example defense: Ensembles of weak defenses are not
  strong.
\newblock In {\em 11th USENIX Workshop on Offensive Technologies (WOOT 17)},
  2017.

\bibitem{athalye2018obfuscated}
Anish Athalye, Nicholas Carlini, and David Wagner.
\newblock Obfuscated gradients give a false sense of security: Circumventing
  defenses to adversarial examples.
\newblock {\em arXiv preprint arXiv:1802.00420}, 2018.

\bibitem{uesato2018adversarial}
Jonathan Uesato, Brendan O’Donoghue, Pushmeet Kohli, and Aaron Oord.
\newblock Adversarial risk and the dangers of evaluating against weak attacks.
\newblock In {\em International Conference on Machine Learning}, pages
  5032--5041, 2018.

\bibitem{cohen2019certified}
Jeremy Cohen, Elan Rosenfeld, and Zico Kolter.
\newblock Certified adversarial robustness via randomized smoothing.
\newblock In {\em International Conference on Machine Learning}, pages
  1310--1320, 2019.

\bibitem{dvijotham2018dual}
Krishnamurthy Dvijotham, Robert Stanforth, Sven Gowal, Timothy~A Mann, and
  Pushmeet Kohli.
\newblock A dual approach to scalable verification of deep networks.
\newblock In {\em UAI}, pages 550--559, 2018.

\bibitem{raghunathan2018certified}
Aditi Raghunathan, Jacob Steinhardt, and Percy Liang.
\newblock Certified defenses against adversarial examples.
\newblock {\em arXiv preprint arXiv:1801.09344}, 2018.

\bibitem{wong2018provable}
Eric Wong and Zico Kolter.
\newblock Provable defenses against adversarial examples via the convex outer
  adversarial polytope.
\newblock In {\em International Conference on Machine Learning}, pages
  5283--5292, 2018.

\bibitem{mirman2018differentiable}
Matthew Mirman, Timon Gehr, and Martin Vechev.
\newblock Differentiable abstract interpretation for provably robust neural
  networks.
\newblock In {\em International Conference on Machine Learning}, pages
  3575--3583, 2018.

\bibitem{wang2018efficient}
Shiqi Wang, Kexin Pei, Justin Whitehouse, Junfeng Yang, and Suman Jana.
\newblock Efficient formal safety analysis of neural networks.
\newblock In {\em Advances in Neural Information Processing Systems}, pages
  6367--6377, 2018.

\bibitem{gowal2018effectiveness}
Sven Gowal, Krishnamurthy Dvijotham, Robert Stanforth, Rudy Bunel, Chongli Qin,
  Jonathan Uesato, Timothy Mann, and Pushmeet Kohli.
\newblock On the effectiveness of interval bound propagation for training
  verifiably robust models.
\newblock {\em arXiv preprint arXiv:1810.12715}, 2018.

\bibitem{zhang2019towards}
Huan Zhang, Hongge Chen, Chaowei Xiao, Bo~Li, Duane Boning, and Cho-Jui Hsieh.
\newblock Towards stable and efficient training of verifiably robust neural
  networks.
\newblock {\em arXiv preprint arXiv:1906.06316}, 2019.

\bibitem{balunovic2020adversarial}
Mislav Balunovic and Martin Vechev.
\newblock Adversarial training and provable defenses: Bridging the gap.
\newblock In {\em International Conference on Learning Representations}, 2020.

\bibitem{li2018second}
Bai Li, Changyou Chen, Wenlin Wang, and Lawrence Carin.
\newblock Second-order adversarial attack and certifiable robustness.
\newblock {\em arXiv preprint arXiv:1809.03113}, 2018.

\bibitem{salman2019provably}
Hadi Salman, Jerry Li, Ilya Razenshteyn, Pengchuan Zhang, Huan Zhang, Sebastien
  Bubeck, and Greg Yang.
\newblock Provably robust deep learning via adversarially trained smoothed
  classifiers.
\newblock In {\em Advances in Neural Information Processing Systems}, pages
  11289--11300, 2019.

\bibitem{dvijotham2020framework}
KD~Dvijotham, J~Hayes, B~Balle, Z~Kolter, C~Qin, A~Gyorgy, K~Xiao, S~Gowal, and
  P~Kohli.
\newblock A framework for robustness certification of smoothed classifiers
  using f-divergences.
\newblock In {\em International Conference on Learning Representations}, 2020.

\bibitem{jia2019certified}
Jinyuan Jia, Xiaoyu Cao, Binghui Wang, and Neil~Zhenqiang Gong.
\newblock Certified robustness for top-k predictions against adversarial
  perturbations via randomized smoothing.
\newblock {\em arXiv preprint arXiv:1912.09899}, 2019.

\bibitem{van2014renyi}
Tim Van~Erven and Peter Harremos.
\newblock R{\'e}nyi divergence and kullback-leibler divergence.
\newblock {\em IEEE Transactions on Information Theory}, 60(7):3797--3820,
  2014.

\bibitem{orabona2015optimal}
Francesco Orabona and D{\'a}vid P{\'a}l.
\newblock Optimal non-asymptotic lower bound on the minimax regret of learning
  with expert advice.
\newblock {\em arXiv preprint arXiv:1511.02176}, 2015.

\bibitem{madry2017towards}
Aleksander Madry, Aleksandar Makelov, Ludwig Schmidt, Dimitris Tsipras, and
  Adrian Vladu.
\newblock Towards deep learning models resistant to adversarial attacks.
\newblock {\em arXiv preprint arXiv:1706.06083}, 2017.

\bibitem{dwork2006calibrating}
Cynthia Dwork, Frank McSherry, Kobbi Nissim, and Adam Smith.
\newblock Calibrating noise to sensitivity in private data analysis.
\newblock In {\em Theory of cryptography conference}, pages 265--284. Springer,
  2006.

\bibitem{mironov2017renyi}
Ilya Mironov.
\newblock R{\'e}nyi differential privacy.
\newblock In {\em 2017 IEEE 30th Computer Security Foundations Symposium
  (CSF)}, pages 263--275. IEEE, 2017.

\bibitem{steinke2016between}
Thomas Steinke and Jonathan Ullman.
\newblock Between pure and approximate differential privacy.
\newblock {\em Journal of Privacy and Confidentiality}, 7(2), 2016.

\bibitem{bun2016concentrated}
Mark Bun and Thomas Steinke.
\newblock Concentrated differential privacy: Simplifications, extensions, and
  lower bounds.
\newblock In {\em Theory of Cryptography Conference}, pages 635--658. Springer,
  2016.

\bibitem{steinke2015between}
Thomas Steinke and Jonathan Ullman.
\newblock Between pure and approximate differential privacy.
\newblock {\em arXiv preprint arXiv:1501.06095}, 2015.

\end{thebibliography}

\newpage
\appendix
To make the paper more readable, we first review some definitions about differential privacy \cite{dwork2006calibrating}. 
\begin{definition} \label{def:DP}
Given a data universe $X$, we say
that two datasets $D, D' \subset X$ are neighbors if they differ by only one entry, which is denoted by
$D\sim D'$. A randomized algorithm $\mathcal{M}$ is $(\epsilon,\delta)$-differentially private (DP) if for all neighboring datasets $D, D'$ and all events $S$
the following holds
\begin{equation*}
    P(\mathcal{M}(D)\in S)\leq e^\epsilon P(\mathcal{M}(D')\in S)+\delta.
\end{equation*}
\end{definition}
\begin{definition}
 A randomized algorithm $\mathcal{M}$ is $(\alpha , \epsilon)$-R\'{e}nyi differentially private (DP) if for all neighboring datasets $D, D'$ 
the following holds
\begin{equation*}
    D_\alpha(\mathcal{M}(D)\|\mathcal{M}(D'))\leq \epsilon. 
\end{equation*}
\end{definition}
\section{Omitted Proofs in Section \ref{sec:over} } 
\begin{proof}[Proof of Theorem \ref{thm:connect_le}]
According to Definition \ref{def:max_renyi}, for a fixed $\xb$, we have $\forall \xb'\in \mathbb{B}_p(\xb, r)$ and any $\alpha\in (1, \infty)$,    $$D_\alpha(\mathcal{M}(\xb)||\mathcal{M}(\xb')) < \alpha\epsilon.$$ Therefore, $\mathcal{M}(\cdot)$ satisfies $(\alpha, \alpha\epsilon)$-R\'{e}nyi  DP ($\xb \in D,~  \xb' \in D'$, and, $\|\xb'-\xb\|_p \leq r$). According to the following lemma, {\em i.e.,}
\begin{lemma}[\cite{mironov2017renyi}]
If a randomized mechanism is $(\alpha, \alpha\epsilon)$-R\'{e}nyi DP, then it is $(\alpha\epsilon + \frac{\log(1/\delta)}{\alpha-1}, \delta)$-DP for any $\delta>0$ (we substitude $\epsilon$ with $\alpha\epsilon$),
\end{lemma}
we have $\mathcal{M}(\cdot)$ is $(\alpha\epsilon + \frac{\log(1/\delta)}{1-\alpha}, \delta)$-DP, for all $\alpha \in (1, +\infty)$. Since
\begin{align*}
    \min_{\alpha \in (1, +\infty)} \{\alpha \epsilon + \log(1/\delta)/(\alpha - 1)\} &\stackrel{\beta=\alpha-1}{=}
\min_{\beta \in (0, \infty)}\{\epsilon(1 + \beta + \frac{\log(1/\delta)}{\epsilon \beta}) \} \\
& = \epsilon + 2\sqrt{\log{(1/\delta)}\epsilon}.
\end{align*}
Thus, by the definition of Approximate Differential Privacy (Definition \ref{def:DP}), in total we have for any $\xb$, $\xb'\in \mathbb{B}_p(\xb, r)$, and any event $S$ 
\begin{equation*}
    P(\mathcal{M}(\xb')\in S)\leq e^{\epsilon + 2\sqrt{\log{(1/\delta)}\epsilon}}P(\mathcal{M}(\xb)\in S)+\delta.
\end{equation*}
Thus by Definition \ref{def:pixeldp}, $\mathcal{M}(\cdot)$ is $(\epsilon +2\sqrt{\log{(1/\delta)}\epsilon}, \delta)$ Pixel-DP. 
\end{proof}

\begin{proof}[Proof of Theorem \ref{thm:connect_li}]
Recall that Lemma \ref{lemma:renyi} indicates that we have $\argmax_{y} P(g(\xb) = y) = \argmax_{y'} P(g(\xb') = y')$ as long as $$D_\alpha(g(\xb)||g(\xb')) < \sup_{\alpha>1} -\log{(1 - p_{(1)} - p_{(2)} + 2(\frac{1}{2}(p_{(1)}^{1-\alpha} + p_{(2)}^{1-\alpha}))^{\frac{1}{1-\alpha}})}.$$
Thus we just need to prove that the above condition holds. 
Since $g(\cdot)$ is $(r, D_{MR}, \|\cdot\|_p, \epsilon)$-robust, for any $\xb$ and $\|\xb'-\xb\|_p\leq r$ we have $$D_\alpha(g(\xb)||g(\xb')) < \alpha\epsilon.$$  If we also have the additional condition: $$\epsilon \leq \sup_{\alpha > 1} -\frac{1}{\alpha}\log{(1 - p_{(1)} - p_{(2)} + 2(\frac{1}{2}(p_{(1)}^{1-\alpha} + p_{(2)}^{1-\alpha}))^{\frac{1}{1-\alpha}})},$$ then $D_\alpha(g(\xb)||g(\xb')) < \sup_{\alpha > 1} -\log{(1 - p_{(1)} - p_{(2)} + 2(\frac{1}{2}(p_{(1)}^{1-\alpha} + p_{(2)}^{1-\alpha}))^{\frac{1}{1-\alpha}})}$. 
Thus, the additional condition to guarantee $\argmax_{y} P(g(\xb) = y) = \argmax_{y'} P(g(\xb') = y')$ can be stated as $\epsilon \leq \sup_{\alpha > 1} -\frac{1}{\alpha}\log{(1 - p_{(1)} - p_{(2)} + 2(\frac{1}{2}(p_{(1)}^{1-\alpha} + p_{(2)}^{1-\alpha}))^{\frac{1}{1-\alpha}})}$. 
\end{proof}
\section{Omitted Proofs in Section \ref{sec:l2}} 
\begin{proof}[Proof of Theorem \ref{thm:l2_gauss}]
By the postprocessing property we just need to show $\mathcal{M}(\xb)= \xb+\zb$ is $(r, D_{MR}, \|\cdot\|_2, \frac{r^2}{2\sigma^2})$ robust.

    Fix any $x$, we have for any $\xb'\in \mathbb{B}_2(\xb, r)$ and $\alpha\in (1, \infty)$ 
    \begin{align*}
        D_\alpha(\mathcal{M}(\xb)\|\mathcal{M}(\xb') )&=D_\alpha (\mathcal{N}(\xb,\sigma^2I_d)\|\mathcal{N}(\xb',\sigma^2I_d)) \\
        &= \frac{\alpha\|\xb'-\xb\|_2^2}{2\sigma^2}\leq \frac{r^2}{2\sigma^2}.
    \end{align*}
\end{proof}

\begin{proof}[Proof of Theorem~\ref{thm:l2lower}] We first show that, in order to prove Theorem~\ref{thm:l2lower}, we only need to prove Theorem~\ref{thm:formall2lower1}. Then we show that, to prove Theorem~\ref{thm:formall2lower1}, we only need to prove Theorem~\ref{thm:formall2lower2}. Finally, we give a formal proof of Theorem~\ref{thm:formall2lower2}.

\begin{theorem} \label{thm:formall2lower1}
For any $\epsilon \leq O(1)$, if there is a $(r, D_{MR}, \|\cdot\|_2, \epsilon)$ randomized (smoothing) mechanism $\mathcal{M}(\xb)=\xb+\zb: \{0, \frac{r}{2\sqrt{d}}\}^d \mapsto \mathbb{R}^d$ such that for any $\xb\in \{0, \frac{r}{2\sqrt{d}}\}^d$,
\begin{equation*}
    \mathbb{E}[\|\zb\|_\infty]= \mathbb{E}[\|\mathcal{M}(\xb)-\xb\|_\infty]\leq \alpha
\end{equation*}
for some $\alpha \leq O(1)$. Then it must be true that $\alpha \geq \Omega(\frac{r}{\sqrt{\epsilon}})$. 
\end{theorem}
For any $\mathcal{M}(\xb) = \xb + \zb: \mathbb{R}^d\mapsto \mathbb{R}^d$, in Theorem~\ref{thm:formall2lower1}, we only consider the expected $\ell_\infty$-norm of the noise added by $\mathcal{M}(\xb)$ on $\xb \in \{0, \frac{r}{2\sqrt{d}}\}^d$. Thus, the $\alpha$ in Theorem~\ref{thm:formall2lower1} should be less than or equal to the $\alpha$ in Theorem~\ref{thm:l2lower} (on $\xb \in \mathbb{R}^d$). Therefore, the lower bound for the $\alpha$ in Theorem~\ref{thm:formall2lower1} ({\em i.e.,} $\Omega(\frac{r}{\sqrt{\epsilon}})$) is also a lower bound for the $\alpha$ in Theorem~\ref{thm:l2lower}. That is to say, if Theorem~\ref{thm:formall2lower1} holds, then Theorem~\ref{thm:l2lower} also holds true. 

Next, we show that if Theorem~\ref{thm:formall2lower2} holds, then Theorem~\ref{thm:formall2lower1} also holds.

\begin{theorem} \label{thm:formall2lower2}
For any $\epsilon \leq O(1)$, if there is a $(r, D_{MR}, \|\cdot\|_2, \epsilon)$-robust randomized (smoothing) mechanism $\mathcal{M}(\xb)\footnote{This mechanism might not be simply $\xb + \zb$ since it must involve operations to clip the output into $[0, \frac{r}{2\sqrt{d}}]^d$}: \{0, \frac{r}{2\sqrt{d}}\}^d\mapsto [0, \frac{r}{2\sqrt{d}}]^d$ such that for any $\xb\in \{0, \frac{r}{2\sqrt{d}}\}^d$ 
\begin{equation*}
    \mathbb{E}[\|\zb\|_\infty]= \mathbb{E}[\|\mathcal{M}(\xb)-\xb\|_\infty]\leq \alpha
\end{equation*}
for some $\alpha \leq O(1)$. Then it must be true that $\alpha \geq \Omega(\frac{r}{\sqrt{\epsilon}})$. 
\end{theorem}
For any $\mathcal{M}(\xb)=\xb + \zb: \{0, \frac{r}{2\sqrt{d}}\}^d\mapsto \mathbb{R}^d$ considered in Theorem~\ref{thm:formall2lower1}, there exists a $(r, D_{MR}, \|\cdot\|_2, \epsilon)$-robust randomized mechanism $\mathcal{M}''(\xb): \{0, \frac{r}{2\sqrt{d}}\}^d\mapsto [0, \frac{r}{2\sqrt{d}}]^d$ considered in Theorem~\ref{thm:formall2lower2} such that for all $\xb \in \{0, \frac{r}{2\sqrt{d}}\}^d$
\begin{equation*}
   \mathbb{E}[\|\mathcal{M}''(\xb)-\xb\|_\infty]\leq  \mathbb{E}[\|\mathcal{M}(\xb)-\xb\|_\infty].
\end{equation*}
To prove the above statement, we first let $a=\frac{r}{2\sqrt{d}}$ and $\mathcal{M}'(\xb)=\min\{\mathcal{M}(\xb), a\}$, where $\min$ is a coordinate-wise operator. Now we fix the randomness of $\mathcal{M}(\xb)$ (that is $\mathcal{M}(\xb)$ is deterministic), and we assume that $\|\mathcal{M}(\xb)-\xb\|_\infty= |\mathcal{M}_j(\xb)-x_j|$, $\|\mathcal{M}'(\xb)-\xb\|_\infty=|\mathcal{M}_i'(\xb)-x_i|$. If $\mathcal{M}_i(\xb)< a$, then by the definitions, we have $\|\mathcal{M}'(\xb)-x\|_\infty = |\mathcal{M}_i'(\xb)-x_i| = |\mathcal{M}_i(\xb)-x_i| \leq \|\mathcal{M}(\xb)-\xb\|_\infty$. If $\mathcal{M}_i(\xb)\geq  a$, then we have $|\mathcal{M}_i'(\xb)-x_i|=|a-x_i|$. Since $x_i \in \{0, a\}$ and $\mathcal{M}_i(\xb)\geq a$, $|\mathcal{M}_i(\xb)-x_i|\geq|a-x_i|$. $\|\mathcal{M}(\xb)-\xb\|_\infty \geq |\mathcal{M}_i(\xb)-x_i|\geq |a-x_i|$. Thus, $\mathbb{E}[\|\mathcal{M}'(\xb)-\xb\|_\infty]\leq  \mathbb{E}[\|\mathcal{M}(\xb)-\xb\|_\infty]$.

Then, we let $\mathcal{M}''(\xb)=\max\{\mathcal{M}'(\xb), 0\}$ where $\max$ is also a coordinate-wise operator. We can use a similar method to prove that $\mathbb{E}[\|\mathcal{M}''(\xb)-\xb\|_\infty]\leq \mathbb{E}[\|\mathcal{M}'(\xb)-\xb\|_\infty] \leq \mathbb{E}[\|\mathcal{M}(\xb)-\xb\|_\infty]$. Also, we can see that $\mathcal{M}''(\xb)=\max\{0, \min\{\mathcal{M}(\xb), a\}\}$, which means $\mathcal{M}''$ is also $(r, D_{MR}, \|\cdot\|_2, \epsilon)$-robust randomized mechanism due to the postprocessing property. 

Since $\mathbb{E}[\|\mathcal{M}''(\xb)-\xb\|_\infty]\leq \mathbb{E}[\|\mathcal{M}(\xb)-\xb\|_\infty]$, and $\mathcal{M}''(\xb)$ is a randomized mechanism satisfying the conditions in Theorem~\ref{thm:formall2lower2}, the $\alpha$ in Theorem~\ref{thm:formall2lower2} should be less than or equal to the $\alpha$ in Theorem~\ref{thm:formall2lower1}.  
Therefore, the lower bound for the $\alpha$ in Theorem~\ref{thm:formall2lower2} ({\em i.e.,} $\Omega(\frac{r}{\sqrt{\epsilon}})$) is also a lower bound for the $\alpha$ in Theorem~\ref{thm:formall2lower1}. That is to say, if Theorem~\ref{thm:formall2lower2} holds, then Theorem~\ref{thm:formall2lower1} also holds.

{\em Finally, we give a proof of Theorem \ref{thm:formall2lower2}.}

	Since $\mathcal{M}$ is $(r, D_{MR}, \|\cdot\|_2, \epsilon)$-robust on $\{0, \frac{r}{2\sqrt{d}}\}^d$, and for any $\xb_i, \xb_j \in \{0, \frac{r}{2\sqrt{d}}\}^d$, 
	$\|\xb_i-\xb_j\|_2 \leq r$ ({\em i.e.,} $\xb_j \in \mathbb{B}_2(\xb_i, r)$), $\mathcal{M}$ is $(\epsilon + 2\sqrt{\log{(1/\delta)}\epsilon}, \delta)$-PixelDP on $\{0, \frac{r}{2\sqrt{d}}\}^d$, according to Theorem \ref{thm:connect_le}. Thus, we also have $\mathcal{M}$ is $(\epsilon + 2\sqrt{\log{(1/\delta)}\epsilon}, \delta)$ DP on the database $\{0, \frac{r}{2\sqrt{d}}\}^d$. 
	
	Then let us take use of the above condition by connecting the lower bound of the sample complexity to estimate one-way marginals ({\em i.e.,} mean estimation) for DP mechanisms with the lower bound studied in Theorem \ref{thm:formall2lower2}. Suppose an $n$-size dataset $X\in \mathbb{R}^{n\times d}$, the one-way marginal is $h(D)=\frac{1}{n}\sum_{i=1}^n{X_i}$, where $X_i$ is the $i$-th row of $X$. In particular, when $n=1$, one-way marginal is just the data point itself, and thus, the condition in Theorem \ref{thm:formall2lower2} can be rewritten as 
	\begin{equation}
		\mathbb{E}[\|\mathcal{M}(D)-h(D)\|_\infty] \leq \alpha.
	\end{equation}
	
	Based on this connection, we first prove the case where $r=2\sqrt{d}$, and then generalize it to any $r$. For $r=2\sqrt{d}$, the conclusion reduces to $\alpha \geq \Omega(\sqrt{\frac{d}{\epsilon}})$. To prove this, we employ the following lemma, which provides a one-way margin estimation for all DP mechanisms.
		\begin{lemma}[Theorem 1.1 in \cite{steinke2016between}]\label{lemma:one_way}
			For any $\epsilon\leq O(1)$, every $2^{-\Omega(n)}\leq \delta \leq \frac{1}{n^{1+\Omega(1)}}$ and every $\alpha\leq \frac{1}{10}$, if $\mathcal{M}: (\{0, 1\}^d)^n \mapsto [0,1]^d$ is $(\epsilon, \delta)$-DP and $\mathbb{E}[\|\mathcal{M}(D)-h(D)\|_\infty]\leq \alpha$, then  we have 
			$
				n\geq \Omega(\frac{\sqrt{d\log \frac{1}{\delta}}}{\epsilon \alpha}). 
			$
		\end{lemma}
		Setting $n=1, \epsilon=\epsilon+2\sqrt{\epsilon \log \frac{1}{\delta}}$ in Lemma~\ref{lemma:one_way}, we can see that if $\mathbb{E}[\|\mathcal{M}(\xb)-\xb\|_\infty]\leq \alpha$, then we must have 
		$$1\geq \Omega(\frac{\sqrt{d\log \frac{1}{\delta}}}{(\epsilon+2\sqrt{\epsilon \log \frac{1}{\delta}})\alpha})\geq \Omega(\frac{\sqrt{d}}{\sqrt{\alpha^2 \epsilon}}),$$ where the last inequality is due to the fact that $\frac{\sqrt{\log \frac{1}{\delta}}}{\epsilon+2\sqrt{\epsilon \log \frac{1}{\delta}}}\geq \Omega(\frac{1}{\sqrt{\epsilon}})$, since $\epsilon\leq O(1)$. 
    Therefore, we have the following theorem, 
    \begin{theorem}\label{thm:simple_l2bound}
		For any $\epsilon\leq O(1)$, if there is a $(2\sqrt{d}, D_{MR}, \|\cdot\|_2, \epsilon)$-robust randomized mechanism $\mathcal{M}: \{0, 1\}^d \mapsto [0,1]^d$ satisfies that
		for all $\xb\in \{0, 1\}^d$
		\begin{equation}
			\mathbb{E}[\|\mathcal{M}(\xb)- \xb\|_\infty] \leq \alpha,
		\end{equation}
		for some $\alpha \leq O(1)$. Then $1\geq \Omega(\sqrt{\frac{d}{\epsilon \alpha^2}})$, {\em i.e.,} $\alpha \geq \Omega(\sqrt{\frac{d}{\epsilon}})$.
	\end{theorem}
	Apparently, Theorem~\ref{thm:simple_l2bound} is special case of Theorem~\ref{thm:formall2lower2} where $r = 2\sqrt{d}$.
	Now we come back to the proof for any $r$. 
	For any $(r, D_{MR}, \|\cdot\|_2, \epsilon)$-robust $\mathcal{M}(\xb): \{0, \frac{r}{2\sqrt{d}}\}^d\mapsto [0, \frac{r}{2\sqrt{d}}]^d$, we substitute $\frac{2\sqrt{d}}{r}\xb$ with $\tilde\xb \in \{0, 1\}^d$ and construct $\tilde{\mathcal{M}}$ as $\tilde{\mathcal{M}}(\tilde\xb) = \frac{2\sqrt{d}}{r}\mathcal{M}(\xb) \in [0, 1]^d$. Since $\mathcal{M}(\xb)$ satisfies
	$$\mathbb{E}[\|\mathcal{M}(\xb)-\xb\|_\infty] \leq \alpha,$$
	then we have $$\mathbb{E}[\|\tilde{\mathcal{M}}(\tilde\xb)-\tilde\xb\|_\infty] = \mathbb{E}[\|\frac{2\sqrt{d}}{r}\mathcal{M}(\xb)-\frac{2\sqrt{d}}{r}\xb\|_\infty] \leq \frac{2\sqrt{d}}{r}\alpha.$$ 
	We claim that $\tilde{\mathcal{M}}: \{0, 1\}^d\mapsto [0, 1]^d$ is $(2\sqrt{d}, D_{MR}, \|\cdot\|_2, \epsilon)$-robust since $\mathcal{M}: \{0, \frac{r}{2\sqrt{d}}\}^d\mapsto [0, \frac{r}{2\sqrt{d}}]^d$ is $(r, D_{MR}, \|\cdot\|_2, \epsilon)$-robust. This is because $D_{MR}(\tilde{\mathcal{M}}(\tilde\xb)\|\tilde{\mathcal{M}}(\tilde\xb')) =  D_{MR}(\frac{2\sqrt{d}}{r}\mathcal{M}(\xb)\|\frac{2\sqrt{d}}{r}\mathcal{M}(\xb'))$, and $D_{MR}(\frac{2\sqrt{d}}{r}\mathcal{M}(\xb)\|\frac{2\sqrt{d}}{r}\mathcal{M}(\xb')) \leq D_{MR}(\mathcal{M}(\xb)\|\mathcal{M}(\xb'))$ since $D_\alpha(f(\mathcal{M}(\xb))\|f(\mathcal{M}(\xb'))) \leq D_\alpha(\mathcal{M}(\xb)\|\mathcal{M}(\xb'))$ for any $f(\cdot)$.
	
    Considering $\tilde{\mathcal{M}}: \{0, 1\}^d\mapsto [0, 1]^d$ in Theorem~\ref{thm:simple_l2bound} with $\alpha=\frac{2\sqrt{d}}{r}\alpha\leq O(1)$ (because $\mathbb{E}[\|\tilde{\mathcal{M}}(\tilde\xb)-\tilde\xb\|_\infty] \leq \frac{2\sqrt{d}}{r}\alpha$), we have 
	\begin{equation}
		1\geq \Omega(\frac{r}{\sqrt{\epsilon\alpha^2}}), ~~~\mbox{i.e., } \alpha \geq \Omega(\frac{r}{\sqrt{\epsilon}}).
	\end{equation}
	Therefore, Theorem~\ref{thm:formall2lower2} holds true, thus, Theorem~\ref{thm:formall2lower1} also holds true, and Theorem~\ref{thm:l2lower} is proved.
\end{proof}

\section{Omitted Proofs in Section \ref{sec:linfty}}

\begin{proof}[Proof of Theorem \ref{thm:exponential}]
    We first prove that $D_\infty(g(\xb)\|g(\xb')) \leq  \frac{r}{\sigma}$ for all $\xb' \in \mathbb{B}_\infty(\xb, r)$.
    Since $\|\xb'-\xb\|_\infty\leq r$, for any $\yb$, 
	\begin{align}
		\frac{p(\yb - \xb)}{p(\yb - \xb')}=& \frac{\exp(-\frac{\|\yb-\xb\|_\infty}{\sigma} ) }{\exp(-\frac{\|\yb-\xb'\|_\infty}{\sigma} )} 
		\leq& \exp( \frac{\|\yb-\xb'\|_\infty-\|\yb-\xb\|_\infty}{\sigma}) \leq& \exp(\frac{\|\xb'-\xb\|_\infty}{\sigma})\leq \exp(\frac{r}{\sigma}). ~~\nonumber
	\end{align}
	Since $$D_\alpha(g(\xb)\|g(\xb')) < D_\infty(g(\xb)\|g(\xb'))=\mathbb{E}[\log\frac{p(\xb)}{p(\xb')}] \leq \frac{r}{\sigma} < \frac{r}{\sigma}\alpha,$$ $\forall \alpha \in (1, +\infty)$, $g(\cdot)$ is $(r, D_{MR}, \|\cdot\|_\infty, \frac{r}{\sigma})$-robust. Also, based on the following lemma,
	\begin{lemma}[\cite{bun2016concentrated}]\label{thm:infty_mr}
	Let $P$ and $Q$ be two probability distributions satisfying $D_\infty(P\|Q)\leq \epsilon$ and $D_\infty(Q\|P)\leq \epsilon$. Then, $D_\alpha(P\|Q)\leq \frac{1}{2}\epsilon^2\alpha$,
    \end{lemma}
    we have $D_\alpha(g(\xb)\|g(\xb'))\leq \frac{1}{2}(\frac{r}{\sigma})^2\alpha$, {\em i.e.,} $g(\cdot)$ is $(r, D_{MR}, \|\cdot\|_\infty, \frac{r^2}{2\sigma^2})$-robust.
\end{proof}

\begin{proof}[Proof of Corollary \ref{thm:exponential_noise_bound}]
	Define the distribution $D$ on $[0, \infty)$ to be $Z\sim D$, meaning $Z=\|\zb\|_\infty$ for $\zb\sim p(\zb)$, where $p(\zb)$ is defined in Eq.(\ref{eq:exponential}). The probability density function of $D$ is given by 
	\begin{equation*}
		p_D(Z)\propto Z^{d-1}\exp(-\frac{Z}{\sigma}),
	\end{equation*}
	which is obtained by integrating the probability density function in Eq. (\ref{eq:exponential}) over the infinity ball of radius $Z$ with surface area $d2^dZ^{d-1}\propto Z^{d-1}$. $p_D$ is the Gamma distribution with shape $d$ and mean $\sigma$, and thus $\mathbb{E}[Z]=d\sigma$.
\end{proof}

\begin{proof}[Proof of Theorem~\ref{thm:linflower}]
Similar to the proof of Theorem \ref{thm:l2lower}, in order to prove Theorem~\ref{thm:linflower}, we only need to prove the following theorem:
\begin{theorem}\label{thm:formallinf}
For any $\epsilon\leq O(1)$, if there is a $(r, D_{MR}, \|\cdot\|_\infty, \epsilon)$ robust randomized (smoothing) mechanism $\mathcal{M}(\xb): \{0, \frac{r}{2}\}^d\mapsto [0, \frac{r}{2}]^d$ such that for any $\xb\in \{0,\frac{r}{2}\}^d$, the following holds 
\begin{equation*}
    \mathbb{E}[\|\zb\|_\infty]= \mathbb{E}[\|\mathcal{M}(\xb)-\xb\|_\infty] \leq \alpha
\end{equation*}
for some $\alpha \leq O(1)$. Then it must be true that $\alpha\geq \Omega (\frac{r\sqrt{d}}{\sqrt{\epsilon}})$. 
\end{theorem}
	Since $\mathcal{M}$ is $(r, D_{MR}, \|\cdot\|_\infty, \epsilon)$-robust on $\{0, \frac{r}{2}\}^d$, and for any $\xb_i, \xb_j \in \{0, \frac{r}{2}\}^d$, 
	$\|\xb_i-\xb_j\|_\infty \leq r$ ({\em i.e.,} $\xb_j \in \mathbb{B}_\infty(\xb_i, r)$), $\mathcal{M}$ is $(\epsilon + 2\sqrt{\log{(1/\delta)}\epsilon}, \delta)$-PixelDP on $\{0, \frac{r}{2\sqrt{d}}\}^d$, according to Theorem \ref{thm:connect_le}. Thus, we also have $\mathcal{M}$ is $(\epsilon + 2\sqrt{\log{(1/\delta)}\epsilon}, \delta)$ DP on the database $\{0, \frac{r}{2\sqrt{d}}\}^d$.
	
     We first consider the case where $r=2$. By setting $n=1$ and $\epsilon=\epsilon+2\sqrt{\epsilon \log 1/\delta}$ in Lemma~\ref{lemma:one_way}, we have a similar result as in Theorem \ref{thm:simple_l2bound}: 
	    \begin{theorem}\label{thm:simple_linfbound}
		For any $\epsilon\leq O(1)$, if there is a $(2, D_{MR}, \|\cdot\|_\infty, \epsilon)$-robust randomized mechanism $\mathcal{M}: \{0, 1\}^d \mapsto [0,1]^d$ satisfies that
		for all $\xb\in \{0, 1\}^d$
		\begin{equation}
			\mathbb{E}[\|\mathcal{M}(\xb)- \xb\|_\infty] \leq \alpha,
		\end{equation}
		for some $\alpha \leq O(1)$. Then $1\geq \Omega(\sqrt{\frac{d}{\epsilon \alpha^2}})$. 
	\end{theorem}
	
	For general $r$, similar to the proof of Theorem~\ref{thm:formall2lower2}, we substitute $\frac{2}{r}\xb$ with $\tilde\xb \in \{0, 1\}^d$ and construct $\tilde{\mathcal{M}}$ as $\tilde{\mathcal{M}}(\tilde\xb) = \frac{2}{r}\mathcal{M}(\xb) \in [0, 1]^d$. Since $\mathcal{M}(\xb)$ satisfies
	$$\mathbb{E}[\|\mathcal{M}(\xb)-\xb\|_\infty] \leq \alpha,$$
	then we have $$\mathbb{E}[\|\tilde{\mathcal{M}}(\tilde\xb)-\tilde\xb\|_\infty] = \mathbb{E}[\|\frac{2}{r}\mathcal{M}(\xb)-\frac{2}{r}\xb\|_\infty] \leq \frac{2}{r}\alpha.$$ 
	Also, $\tilde{\mathcal{M}}: \{0, 1\}^d\mapsto [0, 1]^d$ is $(2, D_{MR}, \|\cdot\|_\infty, \epsilon)$-robust since $\mathcal{M}: \{0, \frac{r}{2}\}^d\mapsto [0, \frac{r}{2}]^d$ is $(r, D_{MR}, \|\cdot\|_\infty, \epsilon)$-robust. This is because $D_{MR}(\tilde{\mathcal{M}}(\tilde\xb)\|\tilde{\mathcal{M}}(\tilde\xb')) =  D_{MR}(\frac{2}{r}\mathcal{M}(\xb)\|\frac{2}{r}\mathcal{M}(\xb'))$, and $D_{MR}(\frac{2}{r}\mathcal{M}(\xb)\|\frac{2}{r}\mathcal{M}(\xb')) \leq D_{MR}(\mathcal{M}(\xb)\|\mathcal{M}(\xb'))$ since $D_\alpha(f(\mathcal{M}(\xb))\|f(\mathcal{M}(\xb'))) \leq D_\alpha(\mathcal{M}(\xb)\|\mathcal{M}(\xb'))$ for any $f(\cdot)$.
    Thus by Theorem \ref{thm:simple_linfbound} with $\alpha=\frac{2}{r}\alpha$ we have 
	\begin{equation*}
	    1\geq \Omega(\frac{\sqrt{d}}{\sqrt{\epsilon (2/r \alpha)^2}}),
	\end{equation*}
	thus we have Theorem~\ref{thm:formallinf}. 
\end{proof}

\begin{proof}[Proof of Theorem \ref{thm:gauss_linf}] By simple calculation we have 
	$$
		D_{\alpha} (\mathcal{N}(\xb, \frac{dr^2}{2\epsilon}I_d)\|\mathcal{N}(\xb', \frac{dr^2}{2\epsilon}I_d))=\frac{\alpha\epsilon\|\xb-\xb'\|_2^2}{dr^2}
		 \leq \frac{\alpha d\epsilon\|\xb-\xb'\|_\infty^2}{dr^2}\leq \alpha\epsilon. 
	$$
	Therefore, $\mathcal{M}(\xb)=\xb+\zb$ with $\zb\sim \mathcal{N}(0, \frac{dr^2}{2\epsilon}I_d)$ is  $(r, D_{MR}, \|\cdot\|_\infty, \epsilon)$-robust. The bound of $\mathbb{E}[\|\zb\|_\infty]$ can be easily proved by substituting $\sigma$ in $O(\sigma\sqrt{\log d})$ \cite{orabona2015optimal} with $\sigma=\sqrt{\frac{dr^2}{2\epsilon}}$.
\end{proof}
\section{Extension to $\ell_p$-norm robustness for Any $\p\in [2, \infty)$} \label{sec:lp}
In previous sections, we studied $\ell_2$-norm and $\ell_\infty$-norm robustness. As we mentioned earlier, our framework can be applied to general norm. In this section, we will study the general $\ell_p$-norm robustness with $p\geq 2$. Just as the previous sections, here we first investigate the $\ell_p$-norm criteria for assessment. 

\begin{theorem}\label{thm:lplower}
Given $p\geq 2$, for any $\epsilon\leq O(1)$, if there is a $(r, D_{MR}, \|\cdot\|_p, \epsilon)$ randomized (smoothing) mechanism $\mathcal{M}(\xb)=\xb+\zb: \mathbb{R}^d\mapsto \mathbb{R}^d$ such that 
\begin{equation*}
    \mathbb{E}[\|\zb\|_\infty]= \mathbb{E}[\|\mathcal{M}(\xb)-\xb\|_\infty]\leq \alpha
\end{equation*}
for some $\alpha \leq O(1)$. Then it must be true that $\alpha \geq \Omega(\frac{rd^{\frac{1}{2}-\frac{1}{p}}}{\sqrt{\epsilon}})$. Note that when $p\rightarrow\infty$, according to Theorem~\ref{thm:linflower}, $\alpha \geq \Omega(\frac{rd^{\frac{1}{2}}}{\sqrt{\epsilon}})$
\end{theorem}
\begin{proof}[Proof of Theorem \ref{thm:lplower}]
	The proof is also almost the same as that of Theorem \ref{thm:l2lower}. Following the proof of Theorem \ref{thm:l2lower}, we can only constrain on the case where $\mathcal{M}(\xb):  \{0, \frac{r}{2\sqrt[p]{d}}\}^d\mapsto [0, \frac{r}{2\sqrt[p]{d}}]^d$.

Since $\mathcal{M}$ is $(r, D_{MR}, \|\cdot\|_p, \epsilon)$-robust on $\{0, \frac{r}{2\sqrt[p]{d}}\}^d$, and for any $\xb_i, \xb_j \in  \{0, \frac{r}{2\sqrt[p]{d}}\}^d$, 
$\|\xb_i-\xb_j\|_p \leq r$ ({\em i.e.,} $\xb_j \in \mathbb{B}_p(\xb_i, r)$), $\mathcal{M}$ is $(\epsilon + 2\sqrt{\log{(1/\delta)}\epsilon}, \delta)$-PixelDP on $ \{0, \frac{r}{2\sqrt[p]{d}}\}^d$, according to Theorem \ref{thm:connect_le}. Thus, we can also say $\mathcal{M}$ is $(\epsilon + 2\sqrt{\log{(1/\delta)}\epsilon}, \delta)$ DP on the database $ \{0, \frac{r}{2\sqrt[p]{d}}\}^d$.

We first consider the case where $r=2\sqrt[p]{d}$, then we extend to the general case. When $r=2\sqrt[p]{d}$, by setting $n=1$ and $\epsilon=\epsilon+2\sqrt{\epsilon \log 1/\delta}$ in Lemma~\ref{lemma:one_way}, we have the following theorem, similar to Theorem \ref{thm:simple_l2bound}.

    \begin{theorem}
	For any $\epsilon\leq O(1)$, if a $(2\sqrt[p]{d}, D_{MR}, \|\cdot\|_p, \epsilon)$-robust randomized mechanism $\mathcal{M}: \{0, 1\}^d \mapsto [0,1]^d$ satisfies that
		for all $\xb\in \{0, 1\}^d$
		\begin{equation}
			\mathbb{E}[\|\mathcal{M}(\xb)- \xb\|_\infty] \leq \alpha,
		\end{equation}
		for some $\alpha \leq O(1)$. Then $1\geq \Omega(\sqrt{\frac{d}{\epsilon \alpha^2}})$. 
	\end{theorem}
\end{proof}

For general $r$, similar to the proof of Theorem~\ref{thm:formall2lower2}, we substitute $\frac{2\sqrt[p]{d}}{r}\xb$ with $\tilde\xb \in \{0, 1\}^d$ and construct $\tilde{\mathcal{M}}$ as $\tilde{\mathcal{M}}(\tilde\xb) = \frac{2\sqrt[p]{d}}{r}\mathcal{M}(\xb) \in [0, 1]^d$. Since $\mathcal{M}(\xb)$ satisfies
	$$\mathbb{E}[\|\mathcal{M}(\xb)-\xb\|_\infty] \leq \alpha,$$
	then we have $$\mathbb{E}[\|\tilde{\mathcal{M}}(\tilde\xb)-\tilde\xb\|_\infty] = \mathbb{E}[\|\frac{2\sqrt[p]{d}}{r}\mathcal{M}(\xb)-\frac{2\sqrt[p]{d}}{r}\xb\|_\infty] \leq \frac{2\sqrt[p]{d}}{r}\alpha.$$ 
	Also, $\tilde{\mathcal{M}}: \{0, 1\}^d\mapsto [0, 1]^d$ is $(2\sqrt[q]{d}, D_{MR}, \|\cdot\|_p, \epsilon)$-robust since $\mathcal{M}: \{0, \frac{r}{2\sqrt[p]{d}}\}^d\mapsto [0, \frac{r}{2\sqrt[p]{d}}]^d$ is $(r, D_{MR}, \|\cdot\|_p, \epsilon)$-robust. This is because $D_{MR}(\tilde{\mathcal{M}}(\tilde\xb)\|\tilde{\mathcal{M}}(\tilde\xb')) =  D_{MR}(\frac{2\sqrt[p]{d}}{r}\mathcal{M}(\xb)\|\frac{2\sqrt[p]{d}}{r}\mathcal{M}(\xb'))$, and $D_{MR}(\frac{2\sqrt[p]{d}}{r}\mathcal{M}(\xb)\|\frac{2\sqrt[p]{d}}{r}\mathcal{M}(\xb')) \leq D_{MR}(\mathcal{M}(\xb)\|\mathcal{M}(\xb'))$ since $D_\alpha(f(\mathcal{M}(\xb))\|f(\mathcal{M}(\xb'))) \leq D_\alpha(\mathcal{M}(\xb)\|\mathcal{M}(\xb'))$ for any $f(\cdot)$.
  Considering $\tilde{\mathcal{M}}$ in Theorem \ref{thm:lplower} with $\alpha = \frac{2\sqrt[p]{d}}{r}\alpha\leq O(1)$, we have
\begin{equation*}
    1\geq \Omega(\sqrt{\frac{d}{\epsilon(2\sqrt[p]{d}/{r}\alpha)^2}})
\end{equation*}
Thus we have $\alpha\geq \Omega(\frac{r d^{\frac{1}{2}-\frac{1}{p}}} {\sqrt{\epsilon}})$. 
\begin{remark}
First, we can see that when $p=2$, Theorem \ref{thm:lplower} is the same as Theorem \ref{thm:l2lower}. Thus, we can see it as a generalization of the previous theorems. Second, Theorem \ref{thm:lplower} indicates that, to certify a certain extent of robustness, the magnitude of the noise we add should be at least $\Omega(d^{\frac{1}{2}-\frac{1}{p}})$, which can be quite large for high dimensional datasets. This means that for $\ell_p$-norm robustness with $p>2$, as a defensive method, the random smoothing method is not very scalable to high dimensional data, {\em i.e.,}, we can call it as the curse of dimensionality on randomized smoothing for certifying $\ell_p (p\geq2)$ robustness. Note that if $p < 2$, then the $\Omega(d^{\frac{1}{2}-\frac{1}{p}}) \rightarrow 0$ as $d \rightarrow \infty$. Therefore, the lower bound is useful when $p\geq2$. 
\end{remark}
In the following theorem, based on the above criteria for $\ell_p$-norm, {\em we show that the Gaussian mechanism is an appropriate option for certifying $\ell_p$-norm robustness.} 
This is because the gap between the criteria and the magnitude of the additive noise required by the Gaussian mechanism is bounded by $O(\sqrt{\log d})$. 
\begin{theorem}[Gaussian Mechanism for Certifying $\ell_p$-norm robustness]\label{thm:gauss_pnorm}
	Let $r, \epsilon>0$ be some fixed number and  $\mathcal{M}(\xb)=\xb+\zb$ with $\zb \sim \mathcal{N}(0, \frac{d^{1-\frac{2}{p}}r^2}{2\epsilon}I_d)$. Then, $\mathcal{M}(\cdot)$  is $(r, D_{MR}, \|\cdot\|_p, \epsilon)$-robust, and $\mathbb{E}[\|\zb\|_\infty]= \mathbb{E}[\|\mathcal{M}(\xb)-\xb\|_\infty]$ is upper bounded by $O(\frac{rd^{\frac{1}{2}-\frac{1}{p}}\sqrt{\log d}}{\sqrt{\epsilon}})$.
\end{theorem}
\begin{proof}[Proof of Theorem \ref{thm:gauss_pnorm}] By simple calculation we have 
	$$
		D_{\alpha} (\mathcal{N}(\xb, \frac{d^{1-\frac{2}{p}}r^2}{2\epsilon}I_d)\|\mathcal{N}(\xb', \frac{d^{1-\frac{2}{p}}r^2}{2\epsilon}I_d))=\frac{\alpha\epsilon\|\xb-\xb'\|_2^2}{d^{1-\frac{2}{p}}r^2}
		 \leq \frac{\alpha d^{1-\frac{2}{p}}\epsilon\|\xb-\xb'\|_p^2}{d^{1-\frac{2}{p}}r^2}\leq \alpha\epsilon. 
	$$
	Therefore, $\mathcal{M}(\xb)=\xb+\zb$ with $\zb\sim \mathcal{N}(0, \frac{d^{1-\frac{2}{p}}r^2}{2\epsilon}I_d)$ is  $(r, D_{MR}, \|\cdot\|_p, \epsilon)$-robust. The bound of $\mathbb{E}[\|\zb\|_\infty]$ can be easily proved by substituting $\sigma$ in $O(\sigma\sqrt{\log d})$ \cite{orabona2015optimal} with $\sigma=\sqrt{\frac{d^{1-\frac{2}{p}}r^2}{2\epsilon}}$.
\end{proof}

\section{Additional Details \& Results}
\subsection{Numerical Method} We first detail the numerical method for the experiments in the following. The core algorithm is detailed in Alg.~\ref{alg:certification}. We follow \cite{cohen2019certified} to conduct evaluations on 500 testing samples from CIFAR10 and ImageNet, and we set $n=1000$ for CIFAR-10 and $n=10000$ for ImageNet in Alg.~\ref{alg:certification}.
Also, we refer the interested readers to our code for more technical details.
\begin{algorithm}
	\caption{Certifying $\ell_2$/$\ell_\infty$-norm Robustness}
	\label{alg:certification}
	\begin{algorithmic}
		\REQUIRE Input $\xb$, a classifier $f(\cdot)$, parameter $\sigma > 0$, number of samples for estimating confidence interval $n$.
		\STATE Sample $n$ samples from the Gaussian/Exponential mechanism $\{\zb_i\}_{i=1...n}$
		\STATE Compute $c_i = f(\xb + \zb_i)$, and estimate the distribution of $c_i$, {\em i.e.,} $p_j \approx \frac{\#\{c_i = j\}_{i=1...n}}{n}$ 
		\vspace{0.1cm}
		\STATE Note that we compute the multinomial confident intervals for $\{p_j\}$. $p_{(1)}$ is set as the lower bound of the largest one in $\{p_j\}$, and $p_{(2)}$ is set as the upper bound of the second largest one in $\{p_j\}$.
		\IF{Choose the Gaussian mechanism}
		\STATE Compute the robust radius by $r_2 = sup_{\alpha > 1} (-\frac{2\sigma^2}{\alpha} \log{(1 - p_{(1)} - p_{(2)} + 2(\frac{1}{2}(p_{(1)}^{1-\alpha} +  p_{(2)}^{1-\alpha}))^{\frac{1}{1-\alpha}})})^{\frac{1}{2}}$
		\ELSIF{Choose the Exponential mechanism}
		\STATE $r_a = sup_{\alpha > 1} -\frac{\sigma}{\alpha} \log{(1 - p_{(1)} - p_{(2)} + 2(\frac{1}{2}(p_{(1)}^{1-\alpha} +  p_{(2)}^{1-\alpha}))^{\frac{1}{1-\alpha}})}$
		\STATE $r_b = (sup_{\alpha > 1} -\frac{2\sigma^2}{\alpha} \log(1 - p_{(1)} - p_{(2)} + 2(\frac{1}{2}(p_{(1)}^{1-\alpha} +  p_{(2)}^{1-\alpha}))^{\frac{1}{1-\alpha}}))^{\frac{1}{2}}$
		\STATE $r_\infty = \max\{r_a, r_b\}$
		\ENDIF
		\STATE \textbf{Output:} For the Gaussian mechanism, $\ell_2$ robust radius is $r_2$, and $\ell_\infty$ robust radius is $\sqrt{r_2^2/d}$. For the Exponential mechanism, $\ell_\infty$ robust radius is $r_\infty$.
	\end{algorithmic}
\end{algorithm}
Here we highlight the sampling method for the Exponential mechanism, which is not detailed in \cite{li2019certified, cohen2019certified}. Due to the high dimensionality of samples in real world applications, directly sampling $\zb\sim p(\zb)$ as in Eq.~\ref{eq:exponential} by the Markov Chain Monte Carlo (MCMC) algorithm requires a large number of random-walks that can incur high computational cost. To alleviate this issue, we adopt an efficient sampling method from \cite{steinke2015between} that first samples $R$ from $Gamma(d+1, \sigma)$ and then samples $\zb$ from $[-R, R]^d$ uniformly. The complexity of this sampling algorithm is only $O(d)$.

\subsection{Additional Experiment Results}
\paragraph{$\ell_2$-norm Case} In Fig.~\ref{fig:cohen_l2_gauss}, we can see that, although \cite{cohen2019certified} proves a tighter bound than ours, it also certifies approximately $40\sim60\%$ accuracy at $\ell_2~\mbox{radius}~=0.34$ (CIFAR-10, $d=3072$) and $\ell_2~\mbox{radius}~=0.29$ (ImageNet, $d=150568$), {\em i.e.,} $O(1/\sqrt{\log d})$. Even after using the advanced training method in \cite{salman2019provably}, the scale of the robust radii is still $O(1/\sqrt{\log d})$, as shown in Fig.~\ref{fig:salman_l2_gauss}.
\begin{figure}[ht]
    \centering
	\includegraphics[width=0.42\columnwidth]{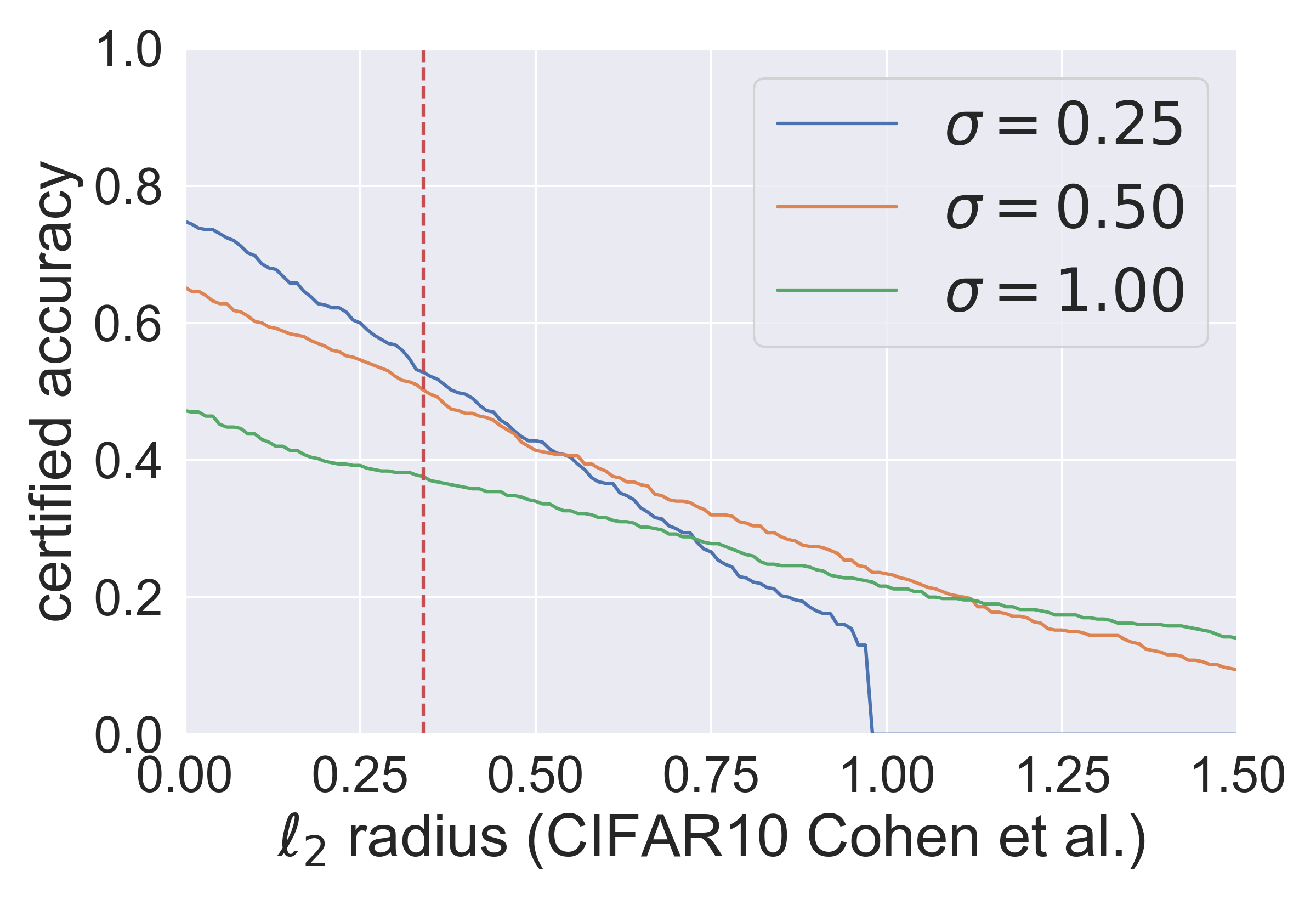}
	\hspace{0.2cm}
	\includegraphics[width=0.42\columnwidth]{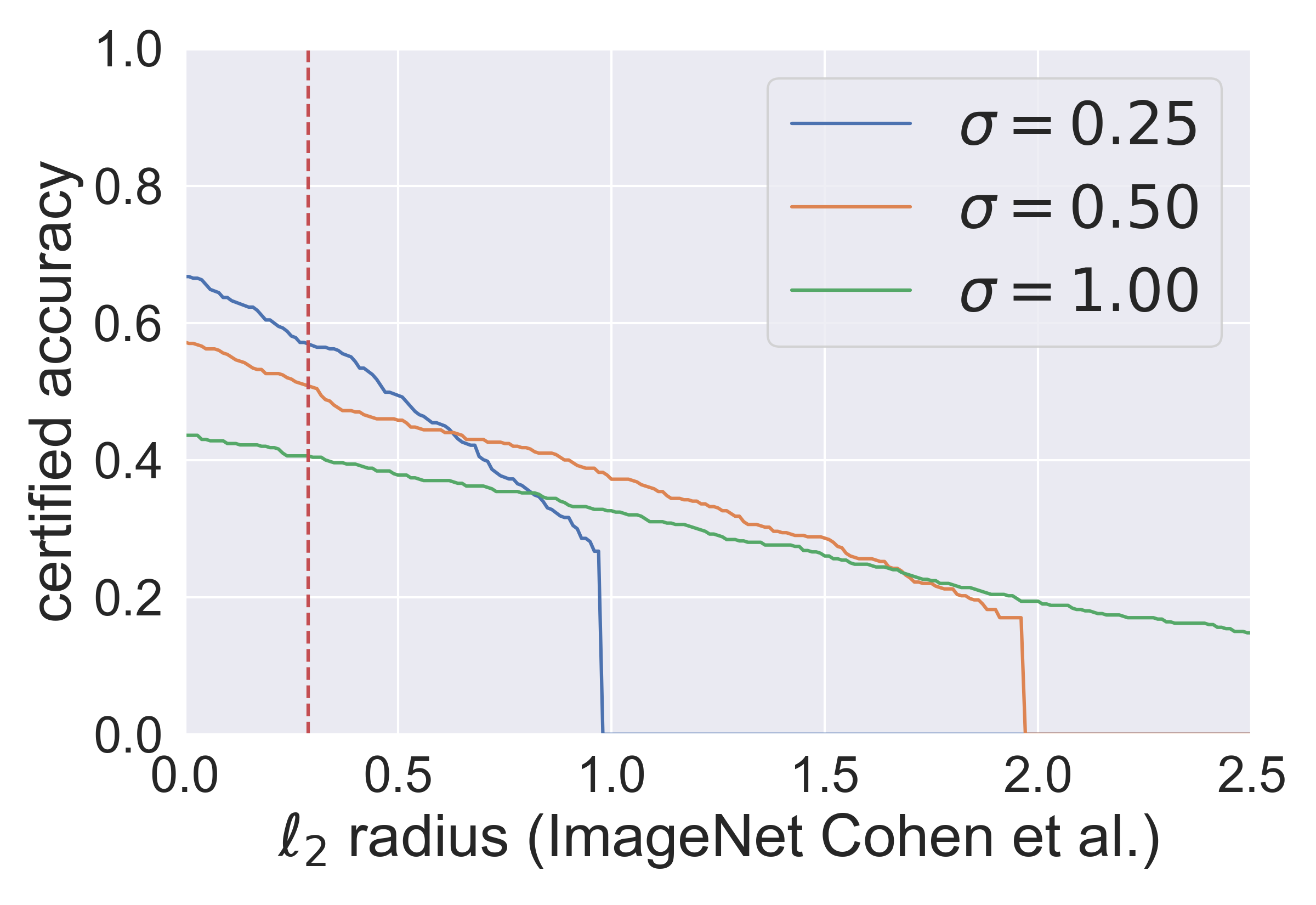}
	\vspace{-0.3cm}
	\caption{Certify $\ell_2$-norm robustness by the Gaussian mechanism \cite{cohen2019certified}: CIFAR10 (left) and ImageNet (right)}
	\label{fig:cohen_l2_gauss}
\end{figure}

\begin{figure}[!htbp]
    \centering
	\includegraphics[width=0.42\columnwidth]{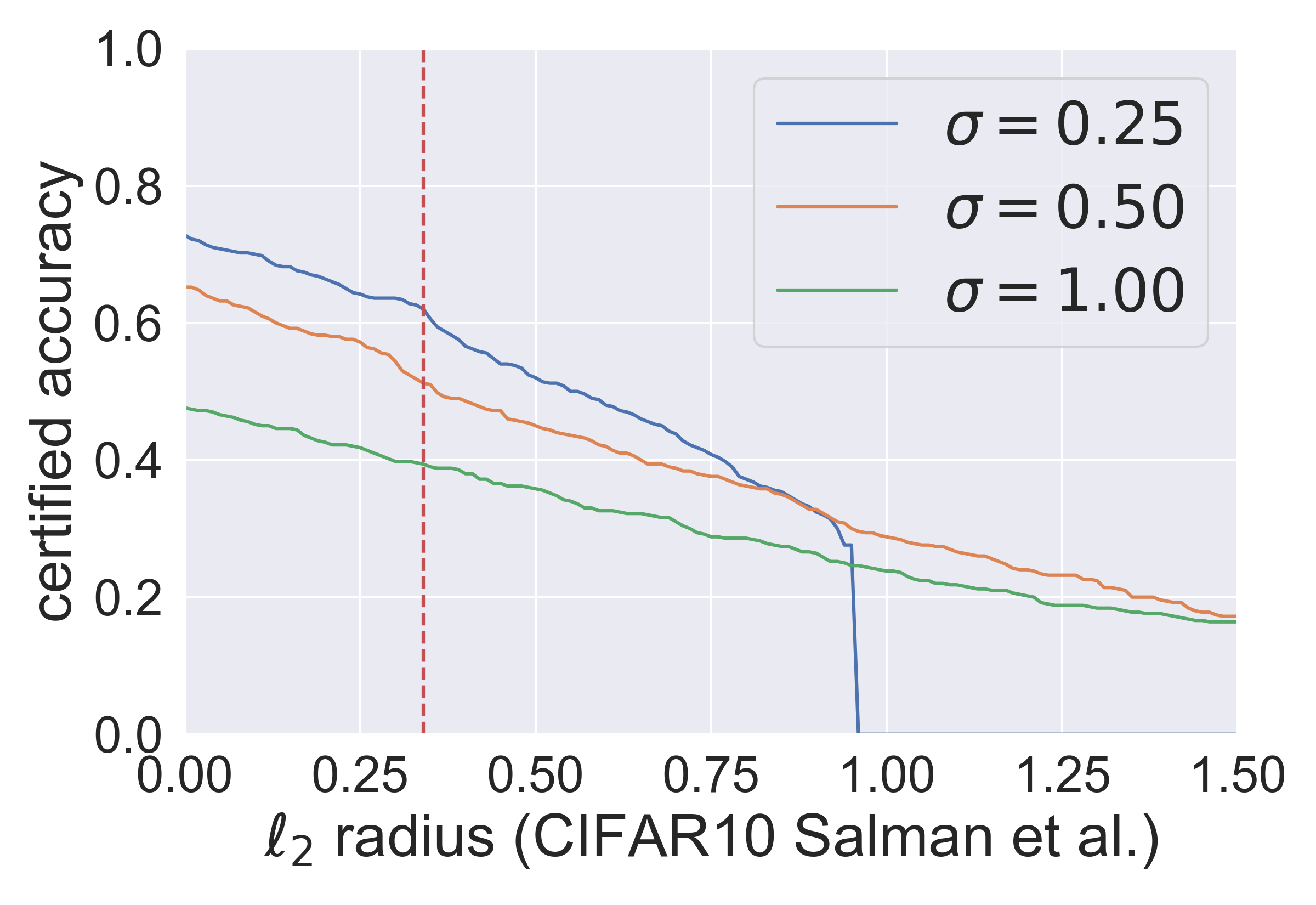}
	\hspace{0.2cm}
	\includegraphics[width=0.42\columnwidth]{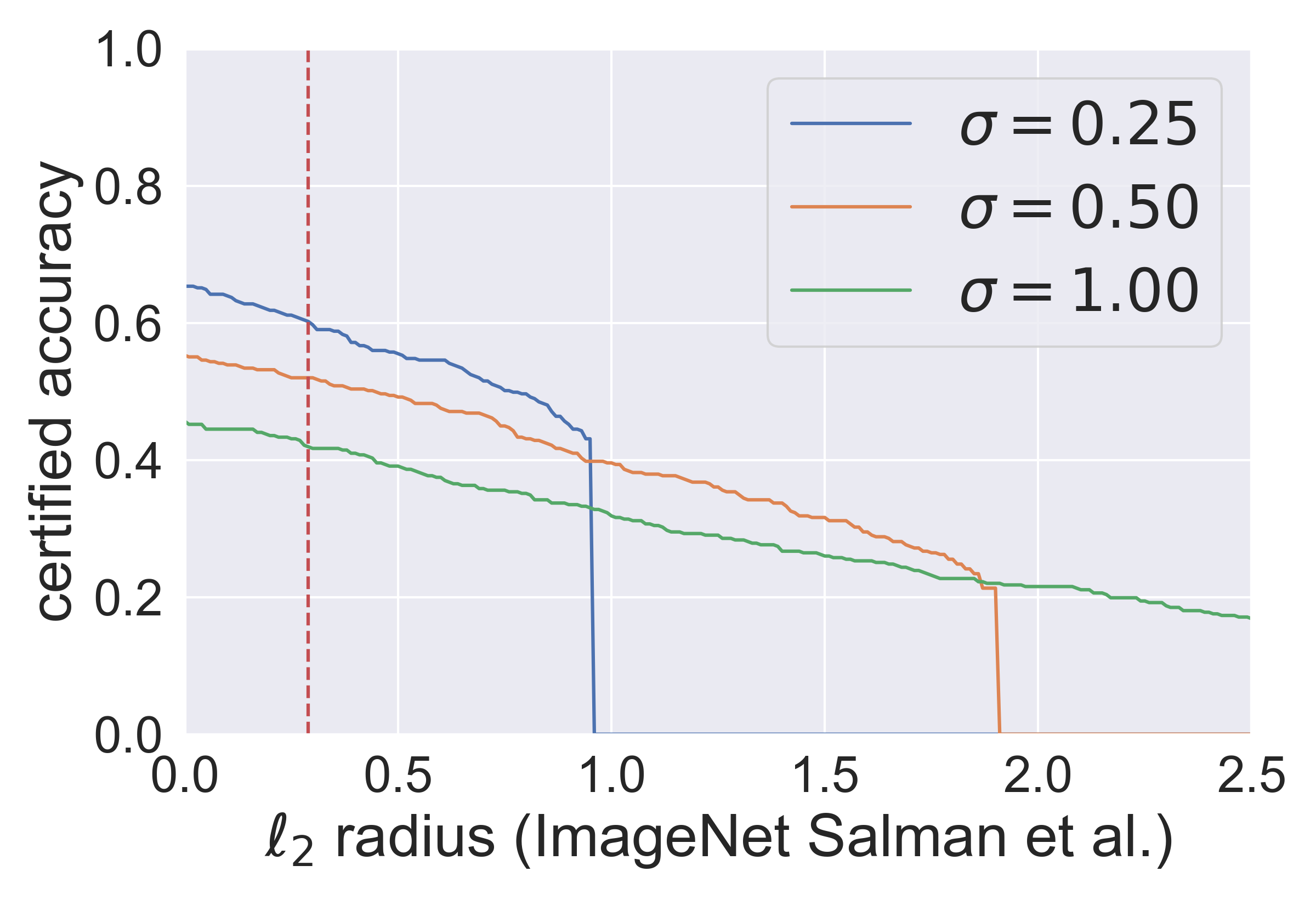}
	\vspace{-0.3cm}
	\caption{Certify $\ell_2$-norm robustness by the Gaussian mechanism and the adversarial training method in \cite{salman2019provably}: CIFAR10 (left) and ImageNet (right)}
	\label{fig:salman_l2_gauss}
\end{figure}

\begin{table}
    \begin{center}
    \scalebox{0.85}{
    \begin{tabular}{c|cc|cc}
    \hline
       \multirow{2}{*}{\textbf{Model}} & \multicolumn{2}{c|}{CIFAR-10} & \multicolumn{2}{c}{(Original) ImageNet} \\
        & $\ell_\infty$ Acc at 2/255 & Standard Acc & $\ell_\infty$ Acc at 1/255 & Standard Acc\\
        \hline \hline
         Cohen et al. \cite{cohen2019certified} (Gaussian) & 47.0\% & 74.8\% ($\sigma=0.25$) & 27.4\% & 57.2\% ($\sigma=0.5$)\\
        \hline
        $D_{MR}$ Framework (Gaussian) & 42.4\% & 69.6\% ($\sigma=0.5$) & 24.6\% & 45.2\% ($\sigma=1.0$)  \\
         \hline
         Wong et al. \cite{wong2018provable} (Single model) & 53.9\% & 68.3\% & - & - \\
         \hline
         IBP \cite{gowal2018effectiveness} & 50.0\% & 70.2\% & - & - \\
         \hline
    \end{tabular}}
    \vspace{0.2cm}
    \caption{Comparing the performance of the Gaussian mechanism with the other works in the $\ell_\infty$ case}\label{tab:comparison}
    \end{center}
\end{table}

\paragraph{$\ell_\infty$-norm Case}
Note that it seems obvious that the Gaussian mechanism is an appropriate mechanism to certify $\ell_2$-norm robustness since \cite{cohen2019certified, li2019certified, salman2019provably} have achieved the state-of-the-art certification results compared with the other methods in the $\ell_2$-norm case. However, in the $\ell_\infty$-norm case, it is a little counterintuitive that the Gaussian mechanism is also an appropriate choice, which performs much better than the Exponential mechanism. In the Table~\ref{tab:comparison}, we compare the $\ell_\infty$-norm certification results of the Gaussian mechanism and the other two representative approaches. Although \cite{cohen2019certified} and the $D_{MR}$ framework perform slightly worse than \cite{wong2018provable} or \cite{gowal2018effectiveness} on CIFAR10, they are more scalable to high-dimensional datasets like ImageNet. So we can say their $\ell_\infty$-norm certification results are comparable. Besides, in Fig.~\ref{fig:cohen_linf_gauss} \& \ref{fig:salman_linf_gauss}, we show that the Gaussian mechanism certifies approximately $40\sim60\%$ accuracy at $\ell_\infty~\mbox{radius}~=6e-3$ on CIFAR-10 and $\ell_\infty~\mbox{radius}~=1.1e-3$ on ImageNet, which are also approximately $O(1/\sqrt{d\log d})$ for both datasets. 

All in all, the empirical results indicate the theorems proved under our framework are valid and very likely to generalize to the other frameworks. 

\begin{figure}[!htbp]
    \centering
	\includegraphics[width=0.42\columnwidth]{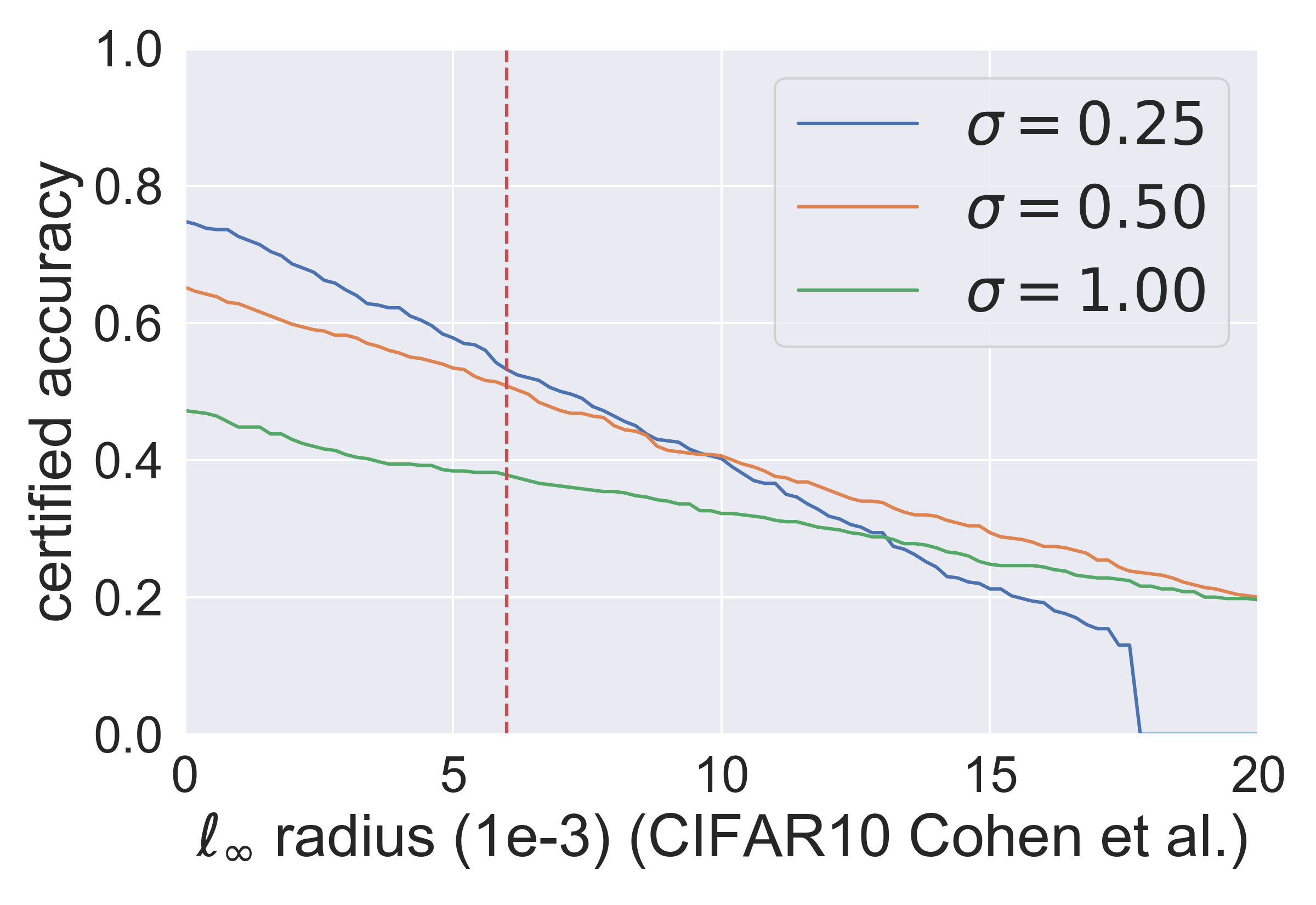}
	\hspace{0.2cm}
	\includegraphics[width=0.42\columnwidth]{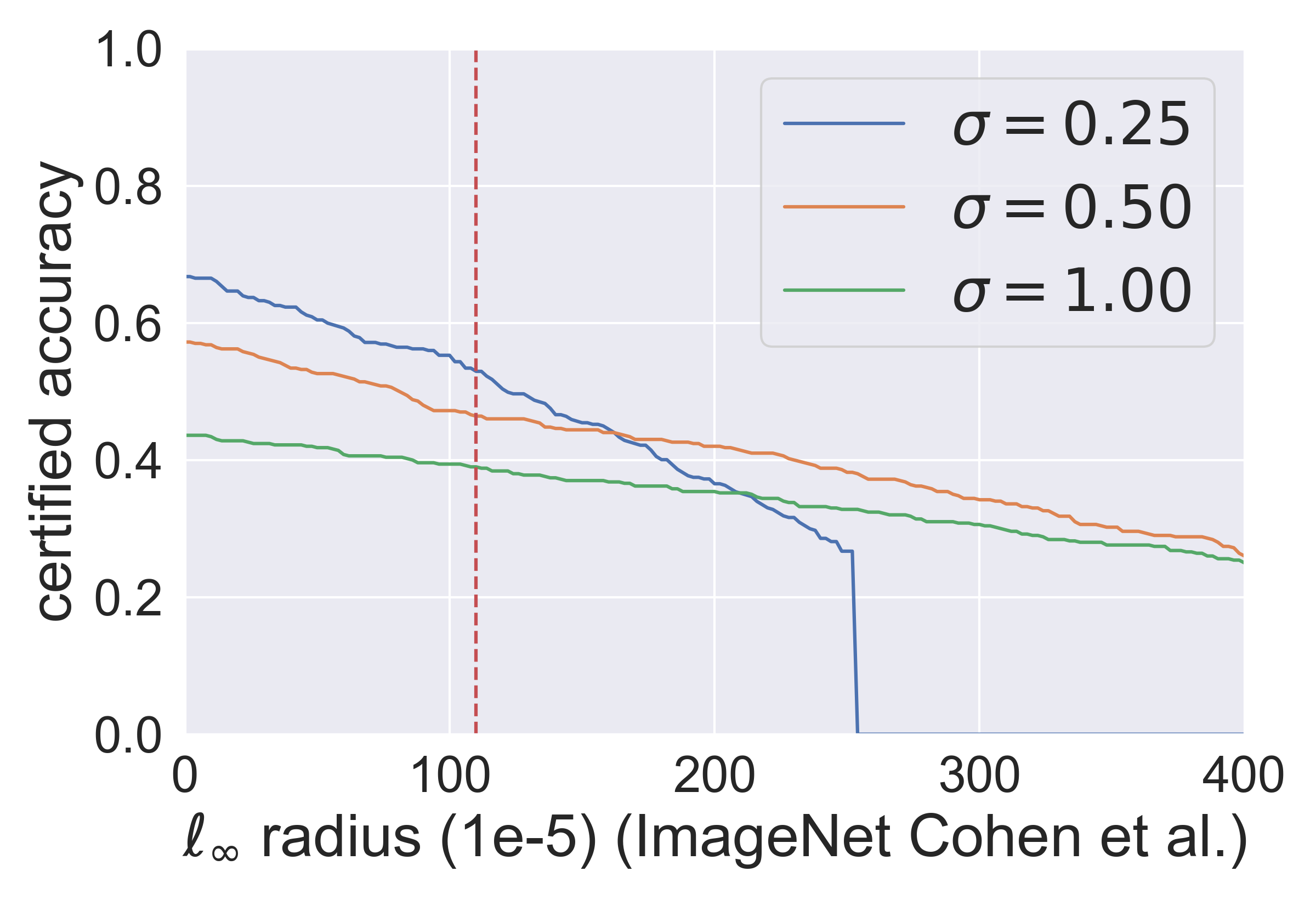}
	\vspace{-0.3cm}
	\caption{Certify $\ell_\infty$-norm robustness by the Gaussian mechanism \cite{cohen2019certified}: CIFAR10 (left) and ImageNet (right)}
	\label{fig:cohen_linf_gauss}
\end{figure}

\begin{figure}[H]
    \centering
	\includegraphics[width=0.42\columnwidth]{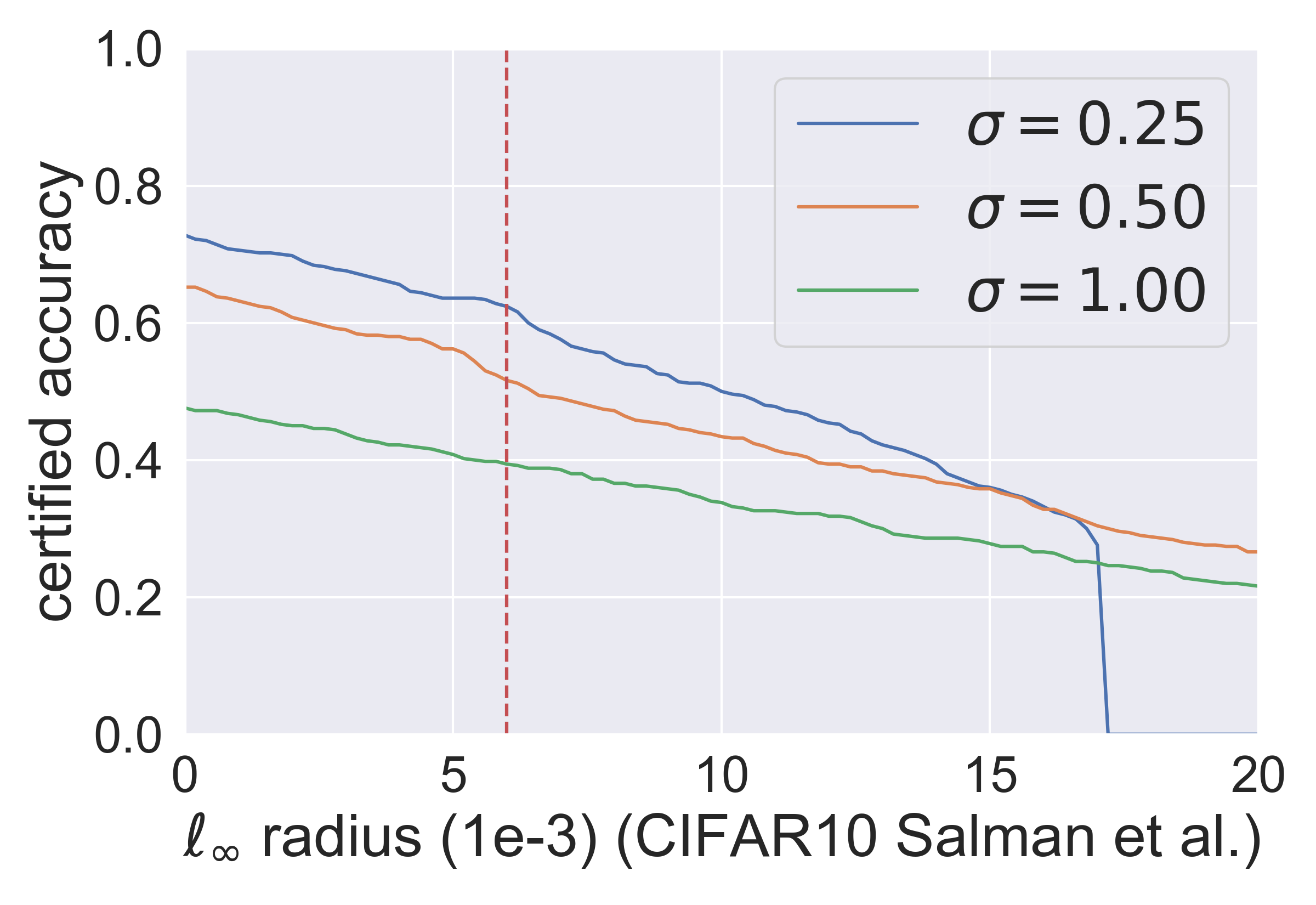}
	\hspace{0.2cm}
	\includegraphics[width=0.42\columnwidth]{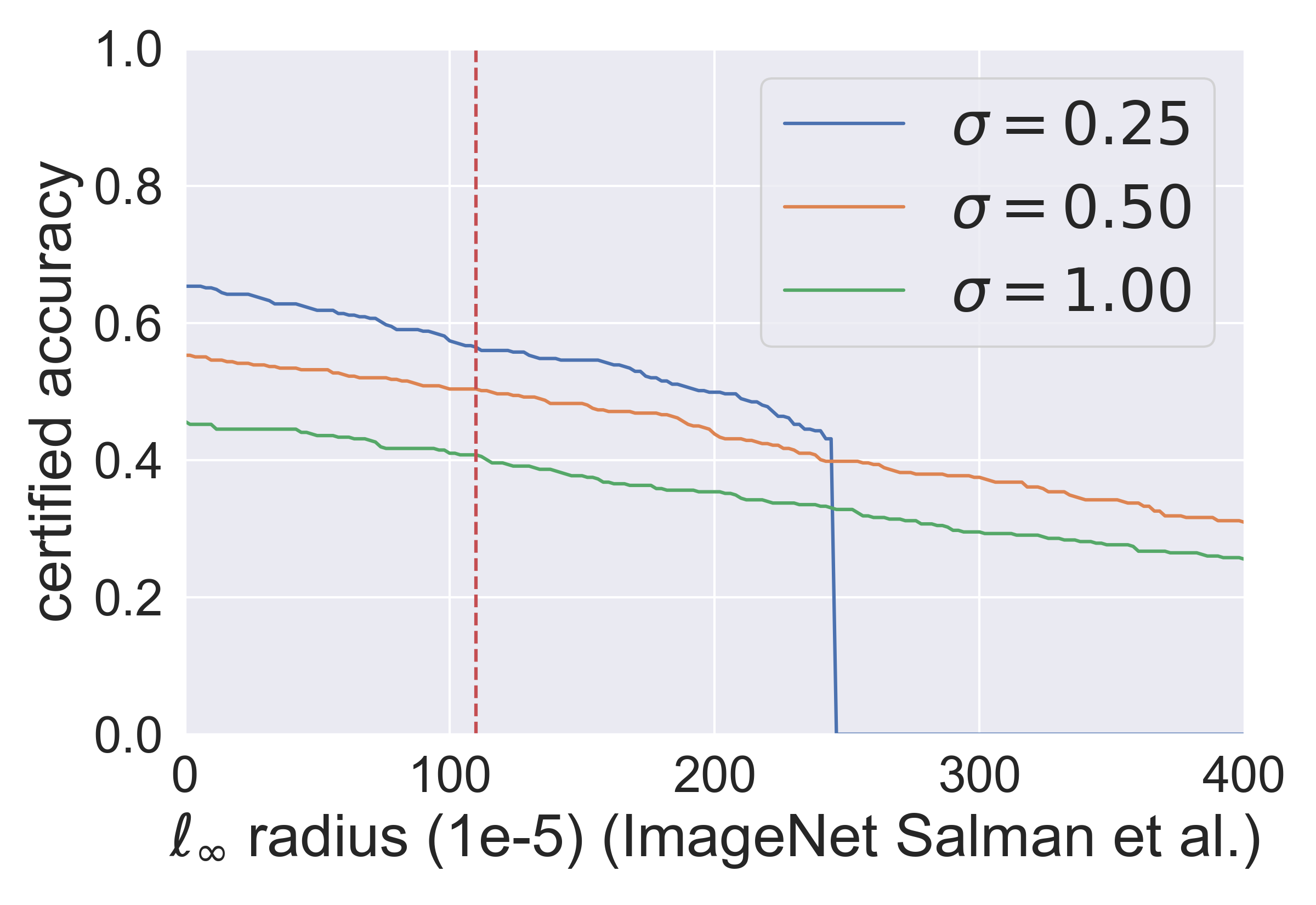}
	\vspace{-0.3cm}
	\caption{Certify $\ell_\infty$-norm robustness by the Gaussian mechanism and the adversarial training method in \cite{salman2019provably}: CIFAR10 (left) and ImageNet (right)}
	\label{fig:salman_linf_gauss}
\end{figure}

\end{document}